\documentclass[sigconf]{acmart}

\AtBeginDocument{%
  \providecommand\BibTeX{{%
    \normalfont B\kern-0.5em{\scshape i\kern-0.25em b}\kern-0.8em\TeX}}}

\copyrightyear{2022}
\acmYear{2022}
\setcopyright{acmcopyright}
\acmConference[KDD '22] {Proceedings of the 28th ACM SIGKDD Conference on Knowledge Discovery and Data Mining}{August 14--18, 2022}{Washington, DC, USA.}
\acmBooktitle{Proceedings of the 28th ACM SIGKDD Conference on Knowledge Discovery and Data Mining (KDD '22), August 14--18, 2022, Washington, DC, USA}
\acmPrice{15.00}
\acmISBN{978-1-4503-9385-0/22/08}
\acmDOI{10.1145/3534678.3539445}

\usepackage{stmaryrd}
\usepackage{paralist}
\usepackage{threeparttable,color}
\usepackage{balance}
\usepackage{bm}
\usepackage{subfig}
\usepackage{graphicx}
\usepackage{amsmath,amsfonts}
\usepackage{booktabs}
\usepackage{multirow}

\newcommand{\vx}{{\mathbf{x}}}
\newcommand{\vy}{{\mathbf{y}}}

\begin{document}

\title{Feature Overcorrelation in  Deep Graph Neural Networks: \\ A New Perspective}

\author{Wei Jin}
\affiliation{%
  \institution{Michigan State University}
}
\email{jinwei2@msu.edu}

\author{Xiaorui Liu}
\affiliation{%
  \institution{Michigan State University}
}
\email{xiaorui@msu.edu}

\author{Yao Ma}
\affiliation{%
  \institution{New Jersey Institute of Technology}
}
\email{yao.ma@njit.edu}

\author{Charu Aggarwal}
\affiliation{%
  \institution{IBM T.J. Watson Research Center}
}
\email{charu@us.ibm.com}
\author{Jiliang Tang}
\affiliation{%
  \institution{ Michigan State University}
}
\email{tangjili@msu.edu}

\begin{abstract}
Recent years have witnessed remarkable success achieved by graph neural networks (GNNs) in many real-world applications such as recommendation and drug discovery. Despite the success, oversmoothing has been identified as one of the key issues which limit the performance of deep GNNs. It indicates that the learned node representations are highly indistinguishable due to the stacked aggregators. In this paper, we propose a new perspective to look at the performance degradation of deep GNNs, i.e., feature overcorrelation. Through empirical and theoretical study on this matter, we demonstrate the existence of feature overcorrelation in deeper GNNs and reveal potential reasons leading to this issue. To reduce the feature correlation, we propose a general framework DeCorr which can encourage GNNs to encode less redundant information.
Extensive experiments have demonstrated that DeCorr can help enable deeper GNNs and is complementary to existing techniques tackling the oversmoothing issue\footnote{Code  is available at \url{https://github.com/ChandlerBang/DeCorr}.}.

\end{abstract}

\keywords{Graph Neural Networks, Semi-supervised Learning, Deep Models}

\maketitle

\section{introduction}

Graphs describe pairwise relations between entities for real-world data from various domains, which are playing an increasingly important role in many applications, including node classification~\cite{kipf2016semi}, recommender systems~\cite{fan2019graph}, drug discovery~\cite{duvenaud2015convolutional} and single cell analysis~\cite{wen2022graph}. To facilitate these applications, it is particularly important to extract effective representations for graphs. In recent years, graph neural networks (GNNs) have achieved tremendous success in representation learning on graphs~\cite{zhou2018graph-survey, wu2019comprehensive-survey}. Most GNNs follow a message-passing mechanism to learn a node representation by propagating and transforming representations of its neighbors~\cite{gilmer2017neural}, which significantly helps them in capturing the complex information of graph data.

Despite the promising results, it has been observed that deeply stacking GNN layers often results in significant performance deterioration~\cite{li2018deeper-oversmooth,zhao2020pairnorm}. Hence, to enable larger receptive filed and larger model capacity, increasing efforts have been made on developing deeper GNNs~\cite{zhao2020pairnorm, zhou2020towards,liu2020towards-oversmooth,rong2019dropedge,chen2020simple}. Most of them attribute the performance deterioration to the oversmoothing issue. In other words, the learned node representations become highly indistinguishable when stacking many GNN layers. In this work, we observe a different issue, overcorrelation, which indicates that deeply stacking GNN layers renders the learned feature dimensions highly correlated. High correlation indicates high redundancy and less information encoded by the learned dimensions, which can harm the performance of downstream tasks.  

We first systematically study the overcorrelation issue in deeper GNNs by answering three questions: (1) does the overcorrelation issue exist? (2) what contributes to the overcorrelation issue? (3) what is the relationship and difference between overcorrelation and oversmoothing? Through exploring these questions, we find that when stacking more GNN layers, generally feature dimensions become more correlated and node representations become more smooth; however they present distinct patterns. Furthermore, through empirical study and theoretical analysis, we show that overcorrelation can be attributed to both propagation and transformation and we further demonstrate that in the extreme case of oversmoothing, the feature dimensions are definitely overcorrelated but not vice versa. In other words, the overcorrelated feature dimensions does not necessarily indicate oversmoothed node representations. These observations suggest that overcorrelation and oversmoothing are related but not identical. Thus, handling overcorrelation has the potential to provide a new and complementary perspective to train deeper GNNs. 

After validating the existence of the overcorrelation issue and understanding its relation with oversmoothing, we aim to reduce the feature correlation and consequently enrich the encoded information for the representations, thus enabling deeper GNNs. In particular, we propose a general framework, DeCorr, to address the overcorrelation issue by introducing an explicit feature decorrelation component and a mutual information maximization component. The explicit feature decorrelation component directly regularizes the correlation on the learned dimensions while the mutual information maximization component encourages the learned representations to preserve a fraction of information from the input features. Our contributions can be summarized as follows:
\begin{compactenum}[(1)]
\item  We introduce a new perspective in deeper GNNs, i.e., feature overcorrelation, and further deepen our understanding on this issue via empirical experiments and theoretical analysis.
\item We propose a general framework to effectively reduce the feature correlation and encourage deeper GNNs to encode less redundant information.
\item Extensive experiments have demonstrated the proposed framework can help enable deeper GNNs and is complementary to existing techniques tackling the oversmoothing issue.
\end{compactenum}

\section{Background and Related Work}
In this section, we introduce the background and related work about graph neural networks. Before that, we first describe key notations used in this paper. We denote a graph as  $\mathcal{G}=({\mathcal{V}},{\mathcal{E}}, {\bf X})$, where $\mathcal{V}=\{v_1, v_2, ..., v_N\}$ is the set of $N$ nodes, $\mathcal{E} \subseteq{\bf\mathcal{V}\times \bf\mathcal{V}}$ is the set of edges describing the relations between nodes, and ${\bf X} \in \mathbb{R}^{N\times d}$ indicates the node feature matrix with $d$ as the number of features. The graph structure can also be described by an adjacency matrix $\mathbf{A} \in \{0,1\}^{N \times N}$ where $\mathbf{A}_{ij}=1$ indicates the existence of an edge between nodes $v_i$ and $v_j$, otherwise $\mathbf{A}_{ij}=0$. Thus, we can also denote a graph as $\mathcal{G}=({\bf A},{\bf X})$.

\subsection{Graph Neural Networks}
\label{sec:concepts}
A graph neural network model usually consists of several GNN layers, where each layer takes the output of the previous layer as the input. Each GNN layer updates the representations of all nodes by propagating and transforming representations of their neighbors. More specifically, the $l$-th GNN layer can be described as follows:
\begin{equation}
\mathbf{H}_{i,:}^{(l)} =  \operatorname{Transform }\left(\operatorname{Propagate}\left(\mathbf{H}_{j,:}^{(l-1)} \mid v_j \in \mathcal{N}(v_i)\cup \{v_i\}\right)\right),
\label{eq:message_pass}
\end{equation}
where ${\bf H}_{i,:}^{(l)}$ denotes the representation for node $v_i$ after $l$-th GNN layer and $\mathcal{N}(v_i)$ is the set of neighboring nodes of node $v_i$. For a $L$ layer graph neural network model, we adopt ${\bf H}^{(L)}$ as the final representation of all nodes, which can be utilized for downstream tasks. For example, for node classification task, we can calculate the discrete label probability distribution for node $v_i$ as follows:
\begin{equation}
\hat{y}_{v_i}=\operatorname{softmax}\left({\bf H}^{(k)}_{i,:}\right),
\end{equation}
where $\hat{y}_{v_i}[j]$ corresponds to the probability of predicting node $v_i$ as the $j$-th class. Graph Convolution Network (GCN)~\cite{kipf2016semi} is one of the most popular graph neural networks. It implements Eq.~\eqref{eq:message_pass} as:
\begin{equation}
\mathbf{H}^{(l)} = {\sigma} \left(\tilde{\bf D}^{-1 / 2}{\bf \tilde{A}} \tilde{\bf D}^{-1 / 2} {\mathbf H}^{(l-1)} {\mathbf W}^{(l)}\right),
\label{eq:gcn}
\end{equation}
where $\tilde{\bf A}={\bf A}+ {\bf I}$ and $\tilde{\bf D}$ is the diagonal matrix of $\tilde{\bf A}$ with $\tilde{\bf D}_{ii} = 1 + \sum_{j} {\bf A}_{ij}$, ${\bf W}^{(l)}$ is the parameter matrix, and $\sigma(\cdot)$ denotes some activation function. In Eq.~\eqref{eq:gcn}, ${\bf W}^{(l)}$ is utilized for feature transformation while $\tilde{\bf D}^{-1 / 2}{\bf \tilde{A}} \tilde{\bf D}^{-1 / 2}$ is involved in feature propagation.  


\subsection{Related Work}
Recent years have witnessed great success achieved by graph neural networks (GNNs) in graph representation learning, which has tremendously advanced various graph tasks~\cite{yan2018spatial,marcheggiani2018exploiting,zitnik2018modeling,guo2021few,liu2021elastic,wang2022improving}. In general, there are two main families of GNN models, i.e. spectral-based methods and spatial-based methods. The spectral-based GNNs utilize graph convolution based on graph spectral theory~\cite{shuman2013emerging} to learn node representations~\cite{bruna2013spectral, henaff2015deep,ChebNet,kipf2016semi}, while spatial-based GNNs update the node representation by aggregating and transforming information from its neighbors~\cite{gat,hamilton2017inductive, gilmer2017neural}. Furthermore, there is significant progress in self-supervised GNNs~\cite{you2021graph,jin2022automated,wang2022graph} and graph data augmentation~\cite{zhao2021data,ding2022data,zhao2022graph}. For a thorough review of GNNs, we please refer the reader to recent surveys~\cite{zhou2018graph-survey,wu2019comprehensive-survey}. 

However, recent studies have revealed that deeply stacking GNN layers can lead to significant performance deterioration, which is often attributed to the oversmoothing issue~\cite{zhao2020pairnorm,chen2020measuring}, i.e, the learned node representations become highly indistinguishable. To address the oversmoothing issue and enable deeper GNNs, various methods have been proposed~\cite{zhao2020pairnorm,chen2020measuring,zhou2020towards,rong2019dropedge,chen2020simple,li2019deepgcns,li2021training}. As an example, PairNorm~\cite{zhao2020pairnorm} is proposed to keep the total pairwise distance of node representations constant through a normalization layer. Similarly, DGN~\cite{zhou2020towards} also normalizes the node representation by normalizing  each group of similar nodes independently to maintain the group distance ratio and instance information gain. Different from the normalization methods, DropEdge~\cite{rong2019dropedge} proposes to randomly drop a certain amount of edges from the graph, which has been shown to alleviate both oversmoothing and overfitting. However, most of the works targets at solving oversmoothing  while overlooking the feature overcorrelation issue. In this paper, we perform a systematical study on the overcorrelation issue and provide effective solution to tackle it. We also provide the connections between the previous methods and overcorrelation in Section~\ref{sec:revisit}.

\section{Preliminary Study}
In this section, we investigate the issues of overcorrelation in deep graph neural networks through both empirical study and theoretical analysis. We observe that overcorrelation and oversmoothing are different though they could be related. 

\subsection{Overcorrelation and Oversmoothing}
\label{sec:pre-study}

In this subsection, we demonstrate that stacking multiple graph neural network layers can sharply increase the correlation among feature dimensions. We choose one popular correlation measure, pearson correlation coefficient~\cite{benesty2009pearson}, to evaluate the correlation between the learned dimensions in deep GNNs. Specifically, given two vectors ${\bf x}\in \mathbb{R}^{N}$ and ${\bf y} \in  \mathbb{R}^{N}$, the pearson correlation coefficient between them can be formulated as follows:
\begin{equation}
    \rho({\bf x}, {\bf y}) = \frac{\sum_{i=1}^N\left({ x}_{i}-\bar{{x}}\right)\left({ y}_{i}-\bar{{ y}}\right)}{\sqrt{\sum_{i=1}^N\left({ x}_{i}-\bar{{ x}}\right)^{2} \sum_{i=1}^N\left({ y}_{i}-\bar{{ y}}\right)^{2}}},
\end{equation}
where $\bar{x}$ and $\bar{y}$ denote the mean value of ${\bf x}$ and ${\bf y}$, respectively. Essentially, pearson correlation coefficient normalizes the covariance between the two variables and measures how much two variables are linearly related to each other. The value of $ \rho({\bf x}, {\bf y})$ ranges from -1 to 1 -- high absolute values of pearson correlation coefficient indicate that the variables are highly correlated and vice versa. Furthermore, we propose the metric $Corr$ to measure the correlation among all learned dimension pairs in the representation ${\bf X}\in \mathbb{R}^{N\times{d}}$ as,
\begin{equation}
    Corr({\bf X}) = \frac{1}{d(d-1)}\sum_{i\neq{j}} |p({\bf X}_{:,i}, {\bf X}_{:,j})| \quad i, j \in [1,2,\ldots, d],
\end{equation}
where ${\bf X}_{:,i}$ denotes the $i$-th column of ${\bf X}$. Since we are also interested in the oversmoothing issue, we use the metric $SMV$ proposed in~\cite{liu2020towards-oversmooth}, which uses normalized node representations to compute their Euclidean distance:
\begin{align}
    SMV(\mathbf{X}) = \frac{1}{N(N-1)} \sum_{i\neq{j}}  D(\mathbf{X}_{i,:}, \mathbf{X}_{j,:}),
\end{align}
where $D(\cdot,\cdot)$ is the normalized Euclidean distance between two vectors:
\begin{equation}
   D(\vx ,\vy) = \frac{1}{2} \Big \| \frac{\vx}{\|\vx\|} - \frac{\vy}{\|\vy\|} \Big \|_2.
\end{equation}
The smaller $SMV$ is, the smoother the node representations are.
Furthermore, it is worth noting that, (1) both $Corr$ and $SMV$ are in the range of $[0,1]$; and (2) they are different measures from two perspectives -- $Corr$ measures \textit{dimension-wise} correlation while $SMV$ measures \textit{node-wise} smoothness.

\begin{figure}[t]%
     \centering
     \subfloat[Cora]{{\includegraphics[width=0.5\linewidth]{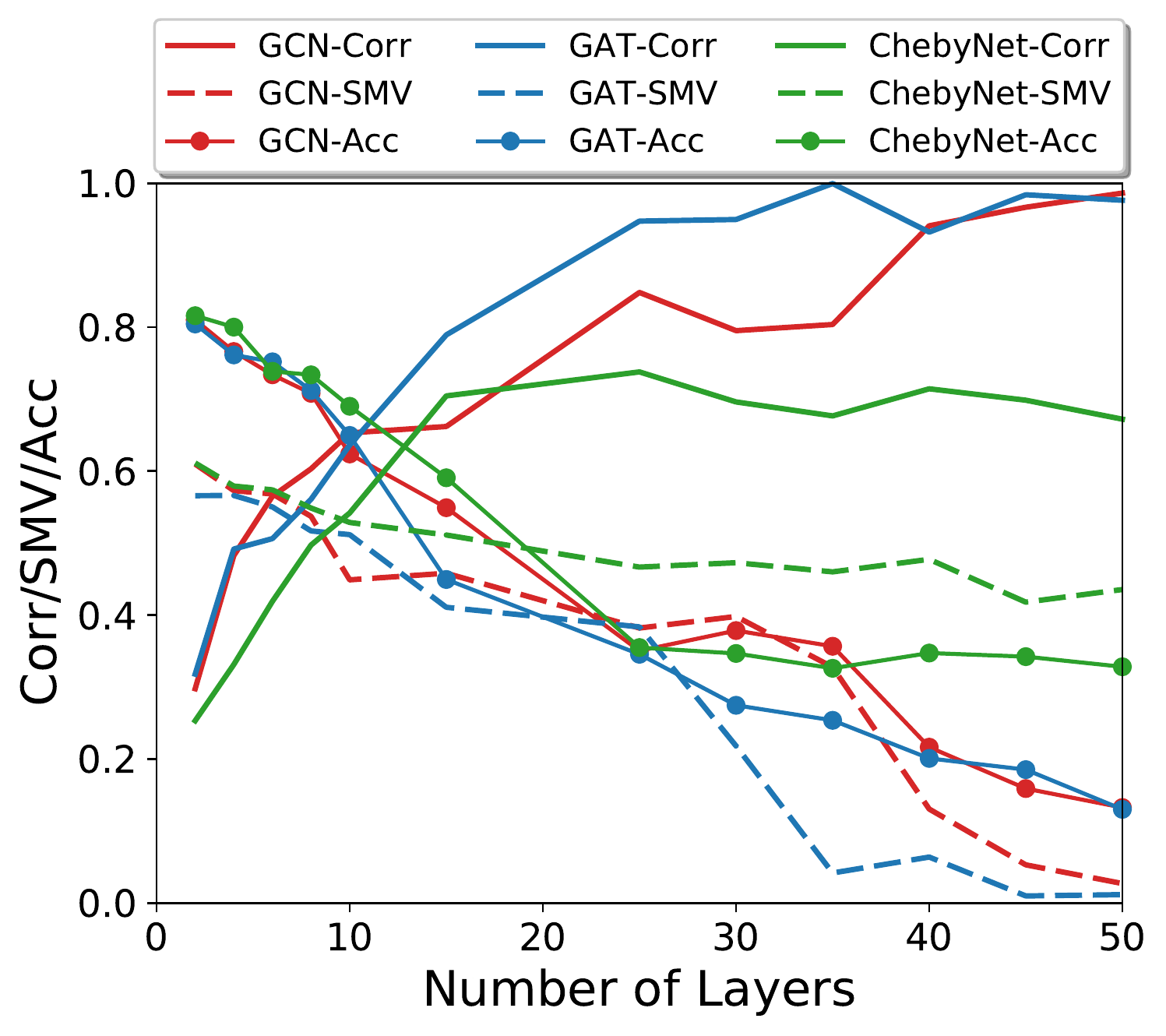} }}%
      \subfloat[Citeseer]{{\includegraphics[width=0.5\linewidth]{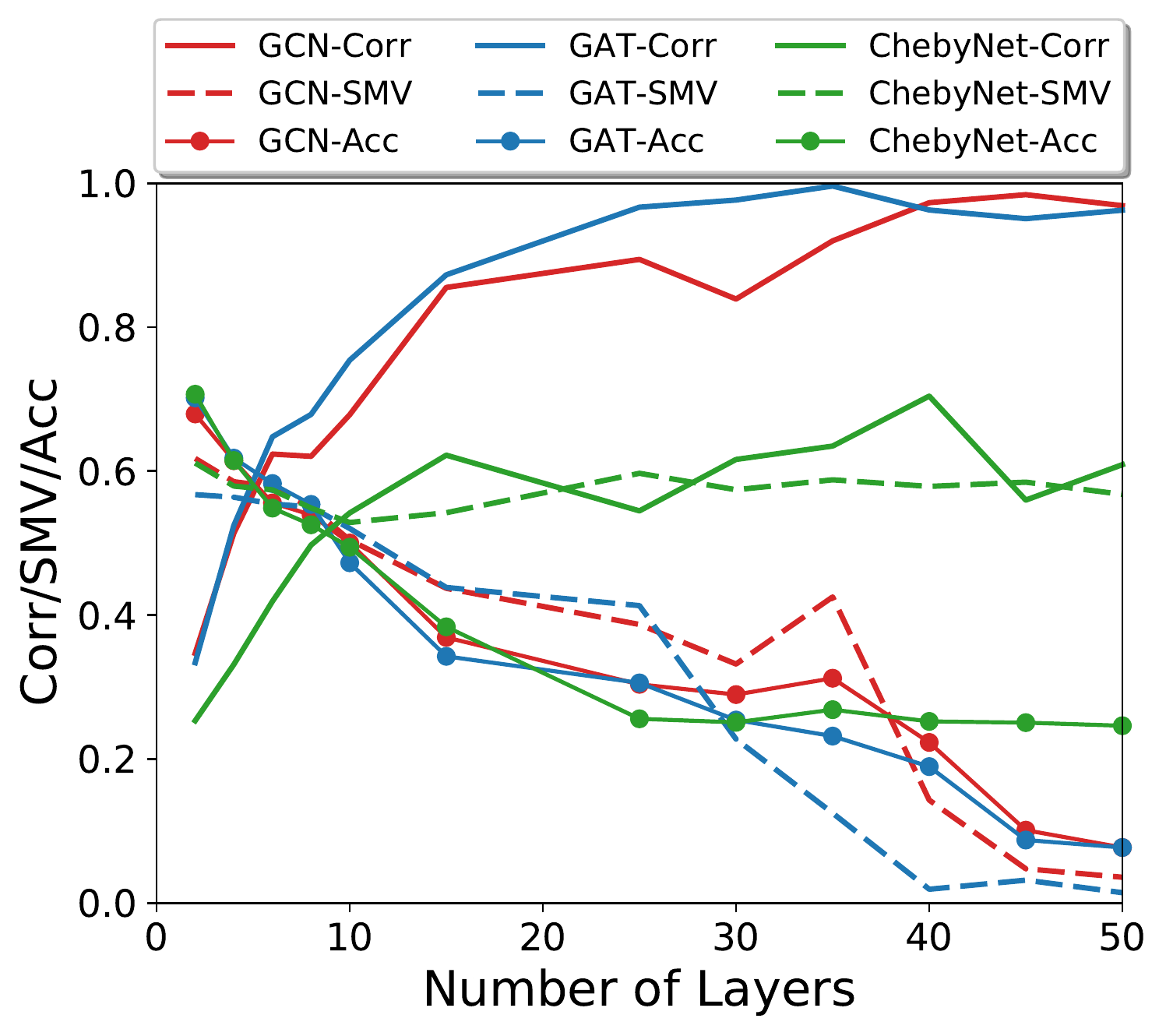} }}%
    \vskip -1em
\caption{Correlation and smoothness of learned representations for three GNNs together with their test accuracy. The larger $Corr$ is, the more correlated the dimensions are; while the smaller $SMV$ is, the smoother the representations are. } 
\label{fig:corr_sim}%
\vskip -1em
\end{figure}

Based on the aforementioned metrics, we investigate the overcorrelation and oversmoothing issues in three representative GNN models, i.e., GCN~\cite{kipf2016semi}, GAT~\cite{gat} and ChebyNet~\cite{ChebNet}. Specifically, we vary the depth of the three models from $2$ to $50$ and calculate the values of correlation ($Corr$) and smoothness ($SMV$) of the final representations. All experiments are repeated $10$ times with different random seeds and we demonstrate the average values in Figure~\ref{fig:corr_sim}. We also plot the results for Pubmed and Coauthor in Figure~\ref{fig:appendix_pubmed_cs} in Append A.1. From the figures, we make following observations:
\begin{compactenum}[(1)]
    \item When the number of layers increases, the correlation among dimensions increases significantly. For example, for 50-layer GCN and GAT, the $Corr$ values on the two datasets are larger than 0.95, which shows extremely high redundancy in the learned dimensions of deep GNNs. We note that \textit{The patterns are even more evident for GCN on Pubmed dataset.}
    \item For the first a few layers (i.e., the number of layers is smaller than 10), the $Corr$ values increase sharply (roughly from 0.3 to 0.6) together with the drop of test accuracy.  However, the $SMV$ values do not change much (roughly from 0.6 to 0.5). That can explain the observation in Section~\ref{sec:mainresults} why methods addressing oversmoothing perform worse than shallow GNNs while our method tackling overcorrelation can outperform them. 
    \item In general, with the increase of the number of layers,  the learned representations becomes more correlated and more smoothing. However they show very different patterns. This observation suggests that overcorrelation and oversmoothing are different though related. Thus, addressing them can help deeper GNNs from different perspectives and can be complementary. This is validated by our empirical results in Section~\ref{sec:mainresults} and Section~\ref{sec:combining}.    
\end{compactenum}

It is evident from this study that (1) deep GNNs suffer the overcorrelation issue, i.e., the learned  dimensions become more and more correlated as the number of layers increases; and (2) overcorrelation and oversmoothing present different patterns with the increase of the number of layers. In the next subsection, we are going to analyze possible factors causing overcorrelation and discuss the problems of overcorelated features.

\subsection{Analysis on Overcorrelation}
Propagation and transformation are two major components in graph neural networks as discussed in Section~\ref{sec:concepts}. In this subsection, we first show that both propagation and transformation can increase feature correlation. Then we discuss potential problems caused by overcorrelated features.

\subsubsection{Propagation Can Lead to Higher Correlation}\label{sec:prop} In this subsection, we will show that propagation in GNNs can cause higher feature correlation from both theoretical analysis and empirical evidence. In essence, the propagation in GNNs will lead to smoother representation as shown in~\cite{li2018deeper-oversmooth,liu2020towards-oversmooth,oono2019graph}; their analysis suggests that applying infinite propagation can make the node representations in a connected graph to be proportional to each other, which we call extreme oversmoothed features in this work. Next we show the dimensions of extreme oversmoothed features are correlated. 
\begin{proposition}
Given an extreme oversmoothed matrix ${\bf X}$ where each row is proportional to each other, we have $Corr({\bf X})$=1.  
\label{pro:smooth}
\end{proposition}
\begin{proof}
Since each row is proportional to each other, each column will also be proportional to each other. We take two  arbitrary dimensions (columns) of ${\bf X}$, and denote them as $[{\bf x}, w{\bf x}]$. The pearson coefficient $\rho({\bf x}, w{\bf x})$ = $\frac{w}{|w|}\cdot\frac{({\bf x-{\bar x}})^{\top}({\bf x-{\bar x}})}{\|{\bf x-{\bar x}}\|\|{\bf x-{\bar x}}\|}$ is either 1 or -1. Since the correlation between any two arbitrary dimensions is 1 or -1, we have $Corr({\bf X})=1$.\footnote{Note that we exclude the case where $\|{\bf x}-{\bf \bar{x}}\|=0$, i.e., ${\bf x}$ is a constant vector, the pearson correlation is undefined. }
\end{proof}

The above proposition indicates that multiple propagation can lead to higher correlation. Though the analysis is only for connected graph, if the training nodes are in the same component, their representations would still be overcorelated and harm the performance of downstream tasks. More importantly, we empirically find that propagation can increase the feature correlation in both connected and disconnected graph. Specifically, we use the full graph (Full) and largest connected component (LCC) of Cora dataset, which originally consists of 78 connected components, to compute the propagation matrix $\hat{\bf{A}}=\tilde{\bf D}^{-1 / 2}{\bf \tilde{A}} \tilde{\bf D}^{-1 / 2}$. Then we apply multiple propagation on randomly generated node features of dimension $100$ whose correlation is very close to 0. We vary the number of propagation ($K$) and illustrate the average values of $Corr(\hat{\bf{A}}^{K}{\bf X})$ for $100$ runs in Figure~\ref{fig:prop}. As we can see from the figure, there is a clear trend that performing multiple propagation can cause uncorrelated features to eventually overcorrelated, no matter whether the graph is connected or not.

Despite that in the extreme case oversmoothing indicates overcorrelation, we show that the opposite of this statement could not hold. Specifically, we have the following proposition:
\begin{proposition}
Given an extreme overcorrelated representation ${\bf X}$, the rows of ${\bf X}$ are not necessarily proportional to each other.
\end{proposition}
\begin{proof}
We take two arbitrary dimensions of ${\bf X}$ and denote them as $[{\bf x}, w{\bf x}+b]$ since they are linearly dependent on each other. We further take two rows from the two dimensions, denoted as $[[x_1, wx_1+b],[x_2, wx_2+b]]$. If the two rows are proportional to each other, they need to satisfy $x_1(wx_2+b)=x_2(wx_1+b)$, which can be written as $bx_1 = bx_2$. When $bx_1 = bx_2$ does not hold, ${\bf X}$ will not be an extreme oversmoothed matrix.
\end{proof}
To demonstrate it more clearly, consider a $2\times2$ matrix with row vectors $v_1=[1, 0]$ and $v_2=[-0.1, 1.1]$. The pearson correlation coefficient between the column vectors is $1$. But the normalized euclidean distance between the row vectors is $0.738$ (the cosine similarity is $-0.09$).

\begin{figure}[t]%
     \centering
     \subfloat[Propagation]{\label{fig:prop}{\includegraphics[width=0.5\linewidth]{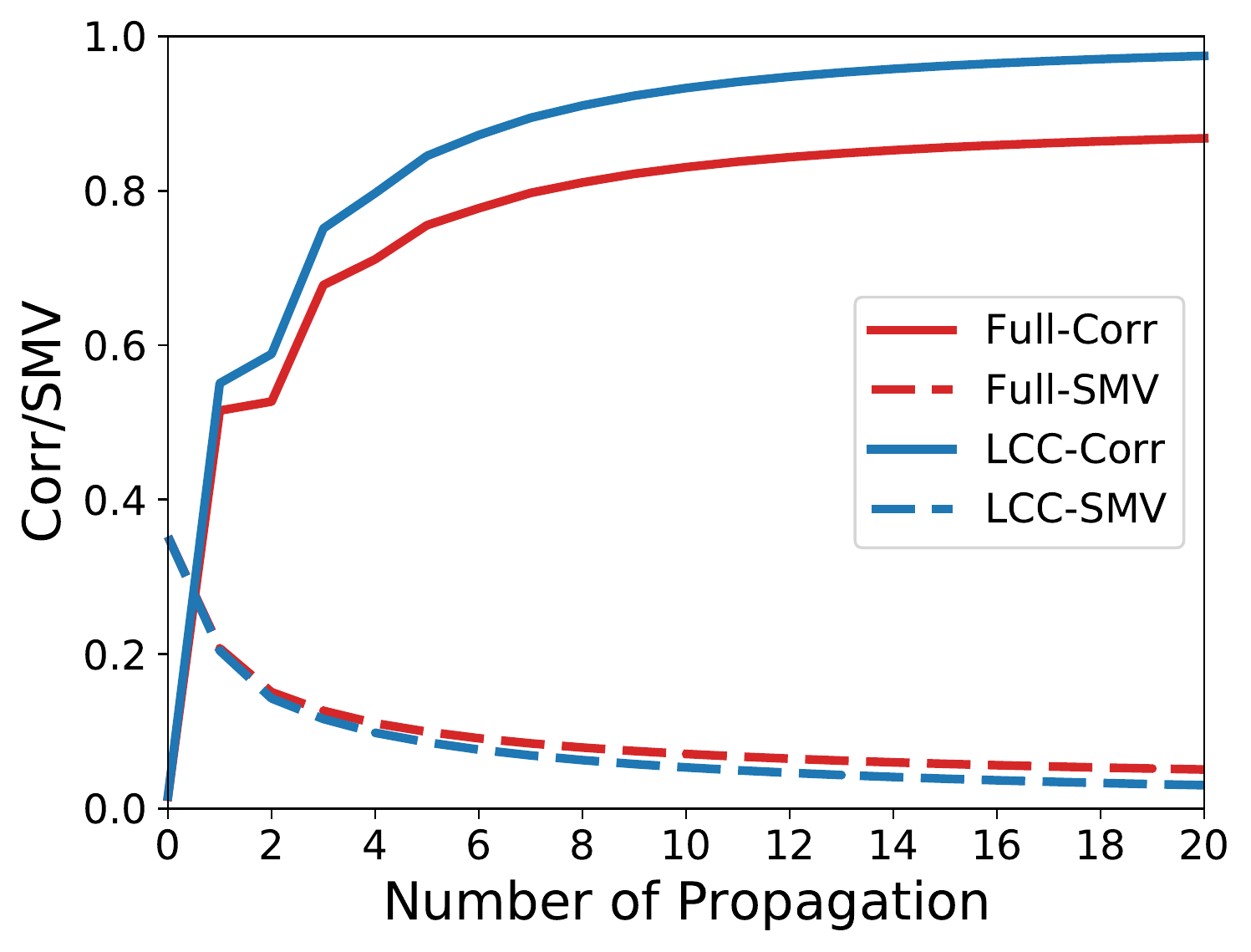} }}%
      \subfloat[Transformation]{\label{fig:trans}{\includegraphics[width=0.5\linewidth]{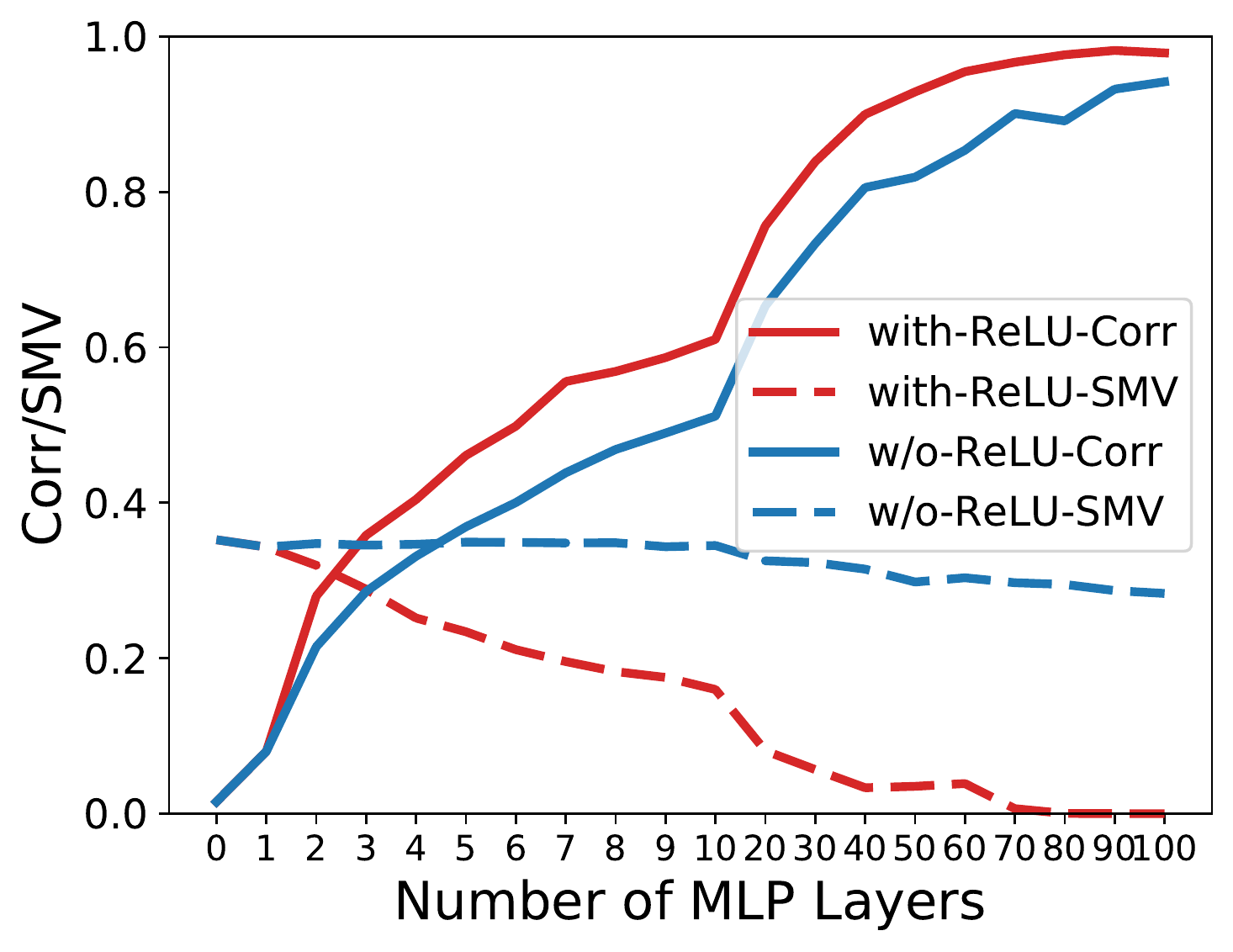} }}%
    \qquad
    \vskip -1em
\caption{The $Corr$ and $SMV$ values on Cora when the number of propagation and transformation increases.}
\vskip -1.5em
\label{fig:prop_trans}%
\end{figure}

\subsubsection{Transformation can Lead to Higher Correlation} In addition to propagation, we also find that transformation can make transformed features more correlated through empirical study. Intuitively, overly stacking transformations will lead to overparameterization, which makes neural networks extract redundant features. \citet{ayinde2019correlation} have empirically identified that the deeper networks have  higher tendency to extract redundant features, i.e., higher correlation among feature dimensions.
Here we also perform empirical study to verify this claim. Specifically, we randomly generate uncorrelated node features of dimension $100$ and apply $K$-layer multi layer perceptron (MLP) with hidden units of $16$. We vary the value of $K$ and plot the $Corr$ values for final representations in Figure~\ref{fig:trans}. Note we do not train the MLP but only focus on the forward pass of the neural network, i.e., the weights of MLP are randomly initialized. 
From the figure we can see that repeatedly applying (linear or non-linear)  transformation will increase the correlation of feature dimensions. Intuitively, this is because transformation linearly combines the feature dimensions and increases the interaction between feature dimensions. As a result, repeated transformation could make each dimension of final representation contain similar amount of information from previous layers and become overcorrelated. On the other hand, backpropagation can to some extent alleviate the overcorrelation issue as the objective downstream task will guide the training process of parameters. However, training very deep neural networks can face other challenges like gradient vanishing/exploding and overfitting, thus failing to effectively reduce the representation correlation. That could also be the reason why trained deep GNNs still exhibit the overcorrelation issue as shown in Figure~\ref{fig:corr_sim}.

\subsubsection{Potential Problems with overcorrelation.} 
The overcorrelation issue indicates that the learned dimensions are redundant and of less efficiency. In this vein, the learned representations may not introduce extra useful information to help train the downstream task and lead to suboptimal performance.  On a separate note, correlated features can bring the overfitting issue and it has been demonstrated that feature decorrelation can lead to better generalization~\cite{cogswell2015reducing,rodriguez2016regularizing}. Hence, it is of great significance to reduce the correlation among feature dimensions.

\subsection{Further Discussions}

In this subsection, we further discuss the difference and relation between overcorrelation and oversmoothing, and revisit existing methods tackling oversmoothing.

\subsubsection{Overcorrelation vs. oversmoothing.} The previous study suggests that overcorrelation and oversmoothing are neither identical nor independent. Their difference can be summarized as: oversmoothing is measured by node-wise smoothness while overcorrelation is measured by dimension-wise correlation. These two measures are essentially different. Correlated feature dimensions do not indicate similar node-wise features. As shown in Figure~\ref{fig:appendix_pubmed_cs}, the $Corr$ value tends to increase while the $SMV$ does not change much on Pubmed and CoauthorCS.  We can conjecture that it is not oversmoothing but overcorrelation that causes the performance degradation on these two datasets. On the other hand, they are also highly related: (1) both overcorrelation and oversmoothing can make the learned representation encode less information and harm the downstream performance; and (2) both of them can be caused by multiple propagation since the extreme case of oversmoothing also suffers overcorrelation as demonstrated in Section~\ref{sec:prop}.

\subsubsection{Revisiting previous methods tackling oversmoothing}\label{sec:revisit} Since we have introduced the overcorrelation issue,  next we revisit the previous methods tackling oversmoothing that have the potential to help alleviate the overcorrelation issue:
\begin{compactenum}[(1)]
\item \textit{DropEdge}~\citep{rong2019dropedge}. DropEdge tackles the oversmoothing problem by randomly dropping edges in the graph. This can be beneficial to correlation reduction: (1) it can weaken the propagation process, thus alleviating overcorrelation; and (2) dropping edges can make the graph more disconnected and further reduce the feature correlation as we showed in Figure~\ref{fig:prop}.

\item \textit{Residual based methods}~\citep{kipf2016semi,chen2020simple,li2019deepgcns}. ResGCN~\citep{kipf2016semi} equips GCN with residual connection and GCNII~\citep{chen2020simple} employs initial residual and identity mapping. Such residual connection bring in additional information from previous layers that can help the final representation encode more useful information and thus alleviate the overcorrelation issue.

\item \textit{Normalization and other methods}~\citep{zhao2020pairnorm,zhou2020towards}. These methods target at pushing GNNs to learn distinct representations. They also can implicitly reduce the feature correlation.
\end{compactenum}

\section{The proposed Framework}
\label{sec:framework}
In this section, we introduce the proposed framework, DeCorr, to tackle the overcorrelation issue. Specifically, the framework consists of two components: (1) explicit feature decorrelation which directly reduces the correlation among feature dimensions; and (2) mutual information maximization which maximizes the mutual information between the input and the representations to enrich the information, thus implicitly making features more independent.

\subsection{Explicit Feature Dimension Decorrelation}
\label{sec:decorr}
In order to decorrelate the learned feature dimensions,  we propose to minimize the correlation among the dimensions of the learned representations. There are many metrics to measure the correlation for variables including linear metrics and non-linear metrics. As discussed in Section~\ref{sec:pre-study}, though measured by linear correlation metric, deep GNNs are shown to have high correlation among dimensions of the learned representations. 
For simplicity, we propose to use the covariance, as a proxy for pearson correlation coefficient, to minimize the correlation among representation dimensions.
Specifically, given a set of feature dimensions  $\{\mathbf{x}_1,\mathbf{x}_2,\ldots,\mathbf{x}_{d}\}$ with $\mathbf{x}_i\in{\mathbb{R}^{N\times{1}}}$, we aim to minimize the following loss function, 
\begin{equation}
\resizebox{0.92\linewidth}{!}{
$\frac{1}{N-1}\left(\sum_{i,j, i \neq j} \left((\mathbf{x}_i - \bar{\mathbf{x}}_i)^{\top}(\mathbf{x}_j - \bar{\mathbf{x}}_j)\right)^2   +  \sum_{i} \left((\mathbf{x}_i - \bar{\mathbf{x}}_i)^{\top}(\mathbf{x}_i - \bar{\mathbf{x}}_i)  - 1 \right)^2\right),$
}
\label{eq:cov}
\end{equation}
where $\bar{\mathbf{x}}_i$ is the vector with all elements as the mean value of ${\mathbf{x}}_i$. In Eq.~\eqref{eq:cov},  minimizing the first term reduces the covariance among different feature dimensions and when the first term is zero, the dimensions will become uncorrelated.  By minimizing the second term, we are pushing the norm of each dimension (after subtracting the mean) to be $1$.  We then rewrite Eq.~\eqref{eq:cov} as the matrix form,
\begin{equation}
\frac{1}{N-1}{\left\|{(\bf{X}-\bf{\bar{X}})^{\top} (\bf{X}-\bf{\bar{X}})}-{\mathbf{I}_{d}}\right\|_{F}^2},
\label{eq:matrix_form}
\end{equation}
where $\bar{\mathbf X}=[\bar{\mathbf{x}}_1, \bar{\mathbf{x}}_2,\ldots, \bar{\mathbf{x}}_d]\in\mathbb{R}^{N\times{d}}$ and $\|\cdot\|_F$ indicates the Frobenius norm. It is worth noting that the gradient of $\left\|{\bf{X}^{\top} \bf{X}}-\mathbf{I}\right\|_{F}^2$ is calculated as $4\bf{X}(\bf{X}^{\top} \bf{X}-\bf{I})$. Thus, the gradient calculation has a complexity of $O(N^{2}d^2)$, which is not scalable when the graph size is extremely large in the real world applications (e.g., millions of nodes). To deal with this issue, instead of using all nodes to calculate the covariance,  we propose to apply Monte Carlo sampling to sample $\sqrt{N}$ nodes with equal probability to estimate the covariance in Eq.~\eqref{eq:cov}. Then the complexity for calculating the gradient reduces to $O(Nd^2)$, which linearly increase with the graph size. In addition,  the sampling strategy injects randomness in the model training process and thus can help achieve better generalization.

Minimizing the loss function in Eq.~\eqref{eq:matrix_form} is analogous to minimizing the following decorrelation loss $\ell_D$ where we normalize the two terms by diving their Frobenius norm in Eq.~\eqref{eq:matrix_form} to make $0\leq\ell_D\leq2$:
\begin{equation}
    \ell_D({\bf X}) = {\left\|\frac{(\bf{X}-\bf{\bar{X}})^{\top} (\bf{X}-\bf{\bar{X}})}{\left\|(\bf{X}-\bf{\bar{X}})^{T} (\bf{X}-\bf{\bar{X}})\right\|_{F}}-\frac{\mathbf{I}_{d}}{\sqrt{d}}\right\|_{F}},
\end{equation}
where $\bf{X}$ can be the output representation matrix for the current GNN layer. Since we hope the representation after each layer can be less correlated, the final decorrelation loss for the explicit representation decorrelation component is formulated as:
\begin{equation}
    \mathcal{L}_D = \sum^{K-1}_{i=1} \ell_D \left({\bf H}^{(i)}\right),
\end{equation}
where ${\bf{H}}^{(i)}$ denotes the hidden representation at the $i$-th layer and $K$ stands for the total number of layers. By minimizing $\mathcal{L}_D$, we explicitly force the representation after each layer to be less correlated, thus mitigating the overcorrelation issue. 

\subsection{Mutual Information Maximization}
\label{sec:mimax}
In Section~\ref{sec:pre-study}, we have demonstrated that the final learned features can be of high redundancy and encode little useful information. To address this issue, in addition to directly constraining on the correlation of features, we propose to further enrich the encoded information by maximizing the mutual information (MI) between the input and the learned features. The motivation comes from independent component analysis (ICA)~\cite{bell1995information-ica}. The ICA principle aims to learn representations with low correlation among dimensions while maximizing the MI between the input and the representation. Since deeper GNNs encode less information in the representations, the MI maximization process can ensure that the learned representations retain a fraction of information from the input even if we stack many layers. Specifically, given two random variables $A$ and $B$, we formulate the MI maximization process as follows:
\begin{equation}
   \max \text{MI}(A, B) = H(A) - H(A|B) = H(B) - H(B|A)
\end{equation}
where $H(\cdot)$ denotes the entropy function and $\text{MI}(A,B)$ measures dependencies between $A$ and $B$. However, maximizing mutual information directly is generally intractable when $A$ or $B$ is obtained through neural networks~\cite{paninski2003estimation}, thus we resort to maximizing a lower bound on $\text{MI}(A, B)$. Specifically, we follow MINE~\cite{belghazi2018mine} to estimate a lower bound of the mutual information by training a classifier to distinguish between sample pairs from the joint distribution $P(A,B)$ and those from $P(A)P(B)$. Formally, this lower bound of mutual information can be described as follows:
\begin{equation}
    \text{MI}(A,B) \geq \mathbb{E}_{P(A,B)} \left[ \mathcal{D}(A,B)\right] -\log \mathbb{E}_{P(A)P(B)}\left[e^{\mathcal{D}(A,B)} \right],
\end{equation}
where $\mathcal{D}(A,B)$ is a binary discriminator. Hence, to maximize the mutual information between $k$-th layer hidden representation ${\bf H}^{(k)}$ and input feature ${\bf X}$, denoted as $\text{MI}({\bf H}^{(k)}, {\bf X})$, we minimize the following objective:
\begin{equation}
    \resizebox{0.88\linewidth}{!}{
    $\ell_M({\bf H}^{(k)},{\bf X}) = - \mathbb{E}_{P({\bf h}^{(k)}_i,{\bf x}_i)} \left[ \mathcal{D}({\bf h}^{(k)}_i,{\bf x}_i)\right] +\log \mathbb{E}_{P({\bf h}^{(k)})P({\bf x})}\left[e^{\mathcal{D}({\bf h}^{(k)}_i,{\bf x}_i)} \right]$
    },
    \label{eq:MI-lower}
\end{equation}
where ${\bf h}^{(k)}_i$ and ${\bf x}_i$ are the hidden representation and input feature for node $v_i$, respectively~\footnote{Here we slightly abuse the notations from previous sections.}; the discriminator $\mathcal{D}(\cdot, \cdot)$ is modeled as:
\begin{equation}
\mathcal{D}({\bf x}_i,{\bf h}^{(k)}_i)=\sigma({\bf x}^{\top}_i{\bf W}{\bf h}^{(k)}_i),
\label{eq:bilinear}
\end{equation}
In practice, in each batch, we sample a set of $\{({\bf h}_i^{(k)}, {\bf x}_i)\}_{i=1}^{B}$ from the joint distribution $P({\bf h}_i^{(k)},{\bf x}_i)$ to estimate the first term in Eq.~\eqref{eq:MI-lower} and then shuffle ${\bf x}_i$ in the batch to generate ``negative pairs'' for estimating the second term.

Although we can apply the above loss function for each layer of deep GNNs as we did in Section~\ref{sec:decorr}, we only apply $\ell_M$ every $t$ layers to accelerate the training process as: 
\begin{equation}
    \mathcal{L}_M = \sum_{i\in[t,2t,3t,...,\frac{K-1}{t} t]} \ell_M \left({\bf H}^{(i)}\right). 
\end{equation}
We empirically observe that a small value of $5$ for $t$ is sufficient to achieve satisfying performance.

\subsection{Objective  and Complexity Analysis}
With the key model components, i.e., explicit representation decorrelation and mutual information maxization, next we first present the overall objective function where we jointly optimize the classification loss along with decorrelation loss and MI maximization loss. Thereafter, we present a complexity analysis and discussion on the proposed framework.

\subsubsection{Overall Objective Function.} The overall objective function can be stated as:
\begin{small}
\begin{align}
\mathcal{L} & =  \mathcal{L}_{class} + \alpha \mathcal{L}_D + \beta \mathcal{L}_M \\
    & = \frac{1}{\left|\mathcal{V}_{L}\right|} \sum_{v_{i} \in \mathcal{V}_{L}} \ell\left(\operatorname{softmax}\left({\mathbf{H}}^{(K)}_{i,:}\right), y_{i}\right) \nonumber \\
     & \quad + \alpha \sum^{K-1}_{i=1} \ell_D \left({\bf H}^{(i)}\right) + \beta \sum_{i\in[t,2t,3t,...,\frac{K-1}{t} t]} \ell_M \left({\bf H}^{(i)}\right)\nonumber
\end{align}
\end{small}where $\mathcal{V}_L$ is the labeled nodes and $y_i$ is the label of node $v_i$; $\alpha$ and $\beta$ are the hyper-parameters that control the contribution of $\mathcal{L}_D$ and $\mathcal{L}_M$, respectively.

\begin{table}[t]
\caption{Node classification accuracy (\%) on different number of layers. (Bold: best)}
\vskip -1em
\scriptsize
\begin{tabular}{@{}c|c|ccc|ccc|ccc@{}}
\toprule
\multirow{2}{*}{\textbf{Dataset}} & \multirow{2}{*}{\textbf{Method}} & \multicolumn{3}{c|}{\textbf{GCN}}              & \multicolumn{3}{c|}{\textbf{GAT}}              & \multicolumn{3}{c}{\textbf{ChebyNet}}         \\ 
                                  &                                  & L2       & L15      & L30      & L2       & L15      & L30      & L2       & L15      & L30     \\ \midrule 
\multirow{6}{*}{Cora}             & None                               & 82.2          & 18.1          & 13.1          & 80.9          & 16.8          & 13.0          & 81.7 & 43.0          & 33.4          \\
                                  & BatchNorm                               & 73.9          & 70.3          & 67.2          & 77.8          & 33.1          & 25.0          & 70.3          & 66.0          & 61.7          \\
                                  & PairNorm                               & 71.0          & 67.2          & 64.3          & 74.4          & 49.6          & 30.2          & 67.6          & 58.2          & 49.7          \\
                                  & DropEdge                         & \textbf{82.8} & 70.5          & 45.4          & {81.5} & 66.3          & 51.0          & {81.5} & 67.1          & 55.7          \\
                                  & DGN                              & 82.0          & 75.2          & {73.2} & 81.1          & 71.8          & 51.3          & 81.5          & \textbf{76.3} & 59.4          \\
                                  & DeCorr                             & 82.2          & \textbf{77.0} & \textbf{73.4}          & \textbf{81.6} & \textbf{76.0} & \textbf{54.3} & \textbf{81.8}          & 73.9          & \textbf{65.4} \\ \midrule
\multirow{6}{*}{Citese.}         & None                               & 70.6          & 15.2          & 9.4           & 70.2          & 22.6          & 7.7           & 67.3          & 38.0          & 28.3          \\
                                  & BatchNorm                               & 51.3          & 46.9          & 47.9          & 61.5          & 28.0          & 21.4          & 51.3          & 38.2          & 37.4          \\
                                  & PairNorm                               & 60.5          & 46.7          & 47.1          & 62.0          & 41.4          & 33.3          & 53.2          & 37.6          & 34.6          \\
                                  & DropEdge                         & 71.7          & 43.3          & 31.6          & 69.8          & 52.6          & 36.1          & 69.8          & 45.8          & 40.7          \\
                                  & DGN                              & 69.5          & 53.1          & 52.6          & 69.3          & 52.6          & 45.6          & 67.3          & 49.3          & 47.0          \\
                                  & DeCorr                             & \textbf{72.1} & \textbf{67.7} & \textbf{67.3} & \textbf{70.6} & \textbf{63.2} & \textbf{46.9} & \textbf{72.6} & \textbf{56.0} & \textbf{53.2} \\ \midrule
\multirow{6}{*}{Pubmed}           & None                               & 79.3          & 22.5          & 18.0          & 77.8          & 37.5          & 18.0          & 78.4          & 49.5          & 43.5          \\
                                  & BatchNorm                               & 74.9          & 73.7          & 70.4          & 76.2          & 56.2          & 46.6          & 73.6          & 68.0          & 69.1          \\
                                  & PairNorm                               & 71.1          & 70.6          & 70.4          & 72.4          & 68.8          & 58.2          & 73.4          & 67.6          & 62.3          \\
                                  & DropEdge                         & 78.8          & 74.0          & 62.1          & 77.4          & 72.3          & 64.7          & 78.7          & 73.3          & 68.4          \\ 
                                  & DGN                              & 79.5          & 76.1          & 76.9          & 77.5          & 75.9          & 73.3          & 78.6          & 71.0          & 70.5          \\
                                  & DeCorr                             & \textbf{79.6} & \textbf{78.1} & \textbf{77.3} & \textbf{78.1} & \textbf{77.5} & \textbf{74.1} & \textbf{78.7} & \textbf{77.0} & \textbf{72.9} \\ \midrule
\multirow{6}{*}{Coauth.}       & None                               & 92.3          & 72.2          & 3.3           & 91.5          & 6.0           & 3.3           & 92.9          & 71.7          & 35.2          \\
                                  & BatchNorm                               & 86.0          & 78.5          & \textbf{84.7}          & 89.4          & 77.7          & 16.7          & 84.1          & 77.2          & 80.7          \\
                                  & PairNorm                               & 77.8          & 69.5          & 64.5          & 85.9          & 53.1          & 48.1          & 79.1          & 51.5          & 57.9          \\
                                  & DropEdge                         & 92.2          & 76.7          & 31.9          & 91.2          & 75.0          & 52.1          & 92.9          & 76.5          & 68.1          \\
                                  & DGN                              & 92.3          & 83.7          & 84.4          & \textbf{91.8} & \textbf{84.5} & 75.5          & 92.7          & 84.0          & 80.4          \\
                                  & DeCorr                             & \textbf{92.4} & \textbf{86.4} & {84.5} & 91.3          & 83.5          & \textbf{77.3} & \textbf{93.0} & \textbf{86.1} & \textbf{81.3} \\ 
\bottomrule
\end{tabular}
\vskip -2em
\label{tab:testacc}
\end{table}
\subsubsection{Complexity Analysis}
We compare the proposed method with vanilla GNNs by analyzing the additional complexity in terms of model parameters and time. For simplicity, we assume that all hidden dimension is $d$ and the input dimension is $d_0$. 

\vskip 0.5em
\noindent\textbf{Model Complexity.} In comparison to vanilla GNNs, the only additional parameters we introduce are the weight matrix ${\bf W}$ in Eq.~\eqref{eq:bilinear} when scoring the agreement for positive/negative samples. Its complexity is $O(d_{0}d)$, which does not depend on the graph size. Since the hidden dimension is usually much smaller than the number of nodes in the graph, the additional model complexity is negligible. 

\vskip 0.5em
\noindent\textbf{Time Complexity.} As shown in Section~\ref{sec:decorr} and \ref{sec:mimax}, the additional computational cost comes from the calculation and backpropagation of the $\mathcal{L}_D$ and $\mathcal{L}_M$ losses. Since we perform Monte Carlo sampling to sample $\sqrt{N}$ nodes, the complexity of calculating $\mathcal{L}_D$ becomes $O(K\sqrt{N}d^2)$ and the complexity of backpropagation for it becomes $O(KNd^2)$; the complexity of calculating $\mathcal{L}_M$ and its gradient is $O(KNd_{0}d)$.  Considering $d$ and $K$ are usually much smaller than the number of nodes $N$, the total additional time complexity becomes $O(KNd^2+KNd_{0}d)$, which increases linearly with the number of nodes.

\section{Experiment}



In this section, we evaluate the effectiveness of the proposed framework under various settings and aim to answer the following research questions: 
\textbf{RQ1.} Can DeCorr help train deeper GNNs? \textbf{RQ2.} By enabling deeper GNNs, is DeCorr able to help GNNs achieve better performance? 
\textbf{RQ3.} Can DeCorr be equipped with methods that tackle oversmoothing and serve as a complementary technique?
\textbf{RQ4.} How do different components affect the performance of the proposed DeCorr?

\subsection{Experimental Settings}
To validate the proposed framework, we conduct experiments on 9 benchmark datasets, including Cora, Citeseer, Pubmed~\cite{sen2008collective}, CoauthorCS~\cite{shchur2018pitfalls}, Chameleon, Texas, Cornell, Wisconsin and Actor~\cite{pei2020geom}. Following~\cite{zhao2020pairnorm,zhou2020towards}, we also create graphs by removing features in validation and test sets for Cora, Citeseer, Pubmed and CoauthorCS. The statistics of these datasets and data splits can be found in Appendix~\ref{sec:dataset}.
We consider three basic GNN models, GCN~\cite{kipf2016semi}, GAT~\cite{gat} and ChebyNet~\cite{ChebNet}, and equip them with the following methods tackling the oversmoothing issue: PairNorm~\cite{zhao2020pairnorm}, BatchNorm~\cite{ioffe2015batch}, DGN~\cite{zhou2020towards}, and DropEdge~\cite{rong2019dropedge}.  We implemented our proposed framework based on the code provided by DGN~\cite{zhou2020towards}, which uses  Pytorch Geometric~\cite{fey2019fast-pyg}. Without specific mention,
we follow previous settings in~\cite{zhao2020pairnorm,zhou2020towards}: train with a maximum of 1000 epochs using the Adam optimizer~\cite{kingmaadam}, run each experiment 5 times and report the average. Detailed parameter settings can be found in Appendix~\ref{sec:param_setting}.

\subsection{Performance of Deeper GNNs}
\label{sec:mainresults}
In this subsection, we integrate GCN, GAT and ChebyNet into our proposed framework and compare the performance with previous methods that tackle oversmoothing under the settings of normal graphs and graphs with missing features. 

\subsubsection{Alleviating the performance drop in deeper GNNs.}\label{sec:521} We aim to study the performance of deeper GNNs when equipped with DeCorr and answer the first research question. Following the previous settings in~\cite{zhou2020towards}, we perform experiments on Cora, Citeseer, Pubmed and CoauthorCS datasets. Note that ``None''  indicates vanilla GNNs without equipping any methods. We report the performance of GNNs with 2/15/30 layers in Table~\ref{tab:testacc} due to space limit while similar patterns are observed in other numbers of layers. As we can see from the table, the proposed DeCorr can greatly improve deeper GNNs. Furthermore, given the same layers, DeCorr consistently achieves the best performance for most cases and significantly slows down the performance drop. For example, on Cora dataset, DeCorr improves 15-layer GCN and 30-layer GCN by a margin of 58.9\% and 60.3\%, respectively. Given the improvement of the performance, the node representations become much more distinguishable than that in vanilla GNNs, thus alleviating the oversmoothing issue just as what the baseline methods do. The above observations indicate that dealing with feature overcorrelation can allow deeper GNNs and even achieves better performance than these only focusing on tackling oversmoothing.

It is worth noting that in most cases the propose framework can also boost the performance of 2-layer GNNs while the strongest baseline, DGN, fail to achieve that and sometimes deteriorate the performance. For instance, on Citeseer dataset, DeCorr improves GCN, GAT and ChebyNet by a margin of 1.5\%, 0.4\% and 5.3\%, respectively. This suggests that decorrelating the learned features is generally helpful for improving the generalization of various models instead of only deeper models.

\subsubsection{Enabling deeper and better GNNs under the missing feature setting}\label{sec:522} In Section~\ref{sec:521}, we have demonstrated the superiority of reducing feature overcorrelation in helping train deeper GNNs. However, the performance achieved by deeper GNNs are not as good as shallow ones (2-layer GNNs) as shown in Table~\ref{tab:testacc}. Then a natural question is: when deeper GNNs are beneficial? To answer this question, we remove the node features in validation and test set following the idea in~\cite{zhao2020pairnorm,zhou2020towards}.  This scenario often happens in real world. For instance, new users in social networks are often lack of profile information while connecting with a few other users~\cite{rashid2008learning}. To learn effective representations for those new users, we would need a larger number of propagation steps to propagate the attribute information from existing users. Hence, under this circumstance, deeper GNNs can behave much better than shallow ones. We vary the number of layers $K$ in $\{1,2,\ldots,10, 15,\ldots,30\}$ and report the performance in Table~\ref{tab:miss}. Note that ``None'' indicates vanilla GNNs without equipping any methods and $\#K$ indicates the number of layers when the model achieves the best performance. Since we empirically find that BatchNorm is necessary in ChebyNet under this setting, we also include the variant of our model which equips with BatchNorm (DGN is also based on BatchNorm). From the table, we make the following two observations: 
\begin{compactenum}[(1)]
    \item Under the missing feature setting, the best performance is always achieved by the deeper models, i.e., the values of $\#K$ are always relatively larger.  This suggests that more propagation steps is necessary to learn the good representation for nodes with missing features.
    \item DeCorr achieves the best performance in 8 out of the 12 cases and significantly outperforms shallow GNNs. For example, on Pubmed dataset, DeCorr achieves improvements of 36.9\%, 33.9\% and 19.5\% on GCN, GAT and ChebyNet, respectively. It further demonstrates the importance of alleviating overcorrelation in developing deeper GNNs.  
\end{compactenum}

\begin{table}[h]
\caption{Test accuracy (\%) on missing feature setting.}
\small
\vskip -1em
\begin{threeparttable}
\begin{tabular}{@{}c|c|cc|cc|cc|cc@{}}
\toprule
\multirow{2}{*}{Model} & \multirow{2}{*}{Method} & \multicolumn{2}{c|}{Cora} & \multicolumn{2}{c|}{Citeseer} & \multicolumn{2}{c|}{Pubmed} & \multicolumn{2}{c}{Coauth.} \\
                       &                         & Acc         &$\#$K        & Acc           &$\#$K          & Acc          &$\#$K         & Acc          &$\#$K          \\ \midrule

\multirow{6}{*}{GCN}   & None                  & 57.3             & 3     & 44.0               & 6       & 36.4              & 4      & 67.3                & 3        \\
                       & BatchNorm                    & 71.8             & 20    & 45.1               & 25      & 70.4              & 30     & 82.7                & 30       \\
                       & PairNorm                    & 65.6             & 20    & 43.6               & 25      & 63.1              & 30     & 63.5                & 4        \\
                       & DropEdge              & 67.0             & 6     & 44.2               & 8       & 69.3              & 6      & 68.6                & 4        \\
                       & DGN                   & \textbf{76.3}    & 20    & \textbf{50.2}      & 30      & 72.0              & 30     & 83.7                & 25       \\
                       & DeCorr                  & 73.8             & 20    & 49.1               & 30      & \textbf{73.3}     & 15     & \textbf{84.3}       & 20       \\ \midrule
\multirow{6}{*}{GAT}   & None                    & 50.1             & 2     & 40.8               & 4       & 38.5              & 4      & 63.7                & 3        \\
                       & BatchNorm                    & 72.7             & 5     & 48.7               & 5       & 60.7              & 4      & 80.5                & 6        \\
                       & PairNorm                    & 68.8             & 8     & 50.3               & 6       & 63.2              & 20     & 66.6                & 3        \\
                       & DropEdge              & 67.2             & 6     & 48.2               & 6       & 67.2              & 6      & 75.1                & 4        \\
                       & DGN                   & \textbf{75.8}    & 8     & \textbf{54.5}               & 5       & 72.3              & 20     & 83.6                & 15       \\
                       & DeCorr                  & 72.8             & 15    &    46.5                &   6      & \textbf{72.4}     & 15     & \textbf{83.7}       & 15       \\ \midrule
\multirow{7}{*}{Cheby.} & None                    & 50.3             & 8     & 31.7               & 4       & 43.7              & 6      & 34.4                & 10       \\
                       & BatchNorm                    & 61.3             & 30    & 35.5               & 6       & 61.6              & 30     & 74.8                & 30       \\
                       & PairNorm                    & 53.7             & 15    & 35.8               & 30      & 53.7              & 25     & 41.5                & 20       \\
                       & DropEdge              & 60.6             & 8     & 35.1               & 4       & 49.3              & 15     & 38.2                & 8        \\
                       & DGN                   & 61.0             & 30    & 35.0               & 30      & 56.3              & 25     & 75.1                & 30       \\
                       & DeCorr                  & 56.0             & 8     & \textbf{35.9}      & 6       & 49.1              & 30     & 48.3                & 10       \\
                       & DeCorr+BN*               & \textbf{62.5}    & 30    & 35.4               & 6       & \textbf{63.2}     & 30     & \textbf{76.5}       & 30       \\ \bottomrule
\end{tabular}
\begin{tablenotes}
   \item[*] DeCorr+BN is the variant when BatchNorm (BN) is equipped with ours. 
\end{tablenotes}
\end{threeparttable}
\label{tab:miss}
\vskip -1em
\end{table}

\subsubsection{Correlation and Smoothness over Training Epochs}
As shown in Table~1, DeCorr achieves significant improvement over other baselines in GAT on Citeseer dataset. Hence, we further investigate the reason behind it. Concretely, we plot the $Corr$, $SMV$, train/val/test accuracy for DGN and DeCorr when they are equipped to 15-layer GAT. The results are shown in Figure~\ref{fig:appendix_L15}. From the figure, we can see that although  DGN can maintain a high $SMV$ value (around 0.6), their performance is still not very satisfying. Moreover, their $Corr$ values are much higher than DeCorr: around 0.6 in  DGN, and around 0.35 in DeCorr. Based on this observation, we believe that overcorrelation is an important issue when enabling deeper GNN that researchers should pay attention to.
\begin{figure}[h]%
     \centering
     \subfloat[DGN]{{\includegraphics[width=0.5\linewidth]{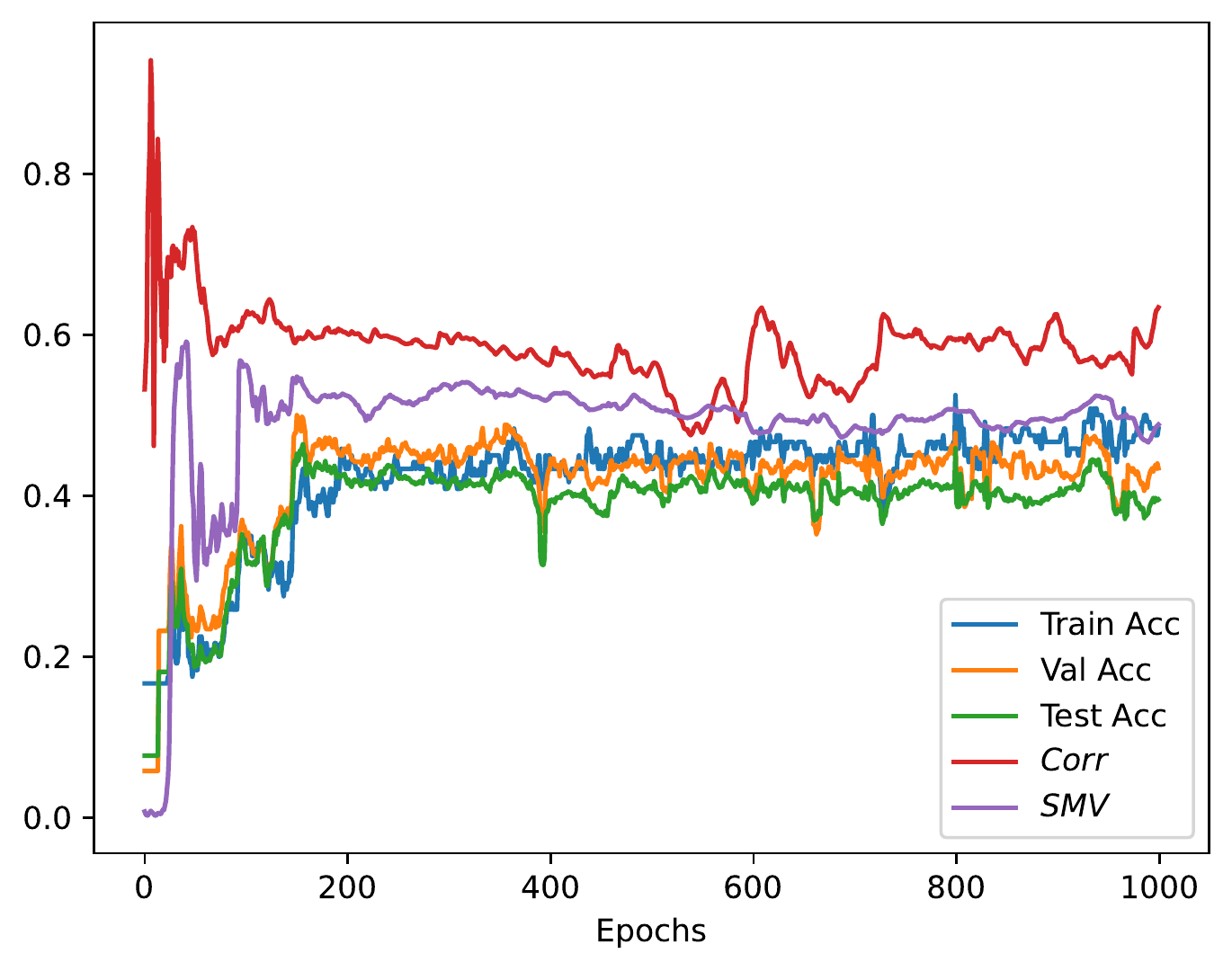}}}%
     \subfloat[DeCorr]{{\includegraphics[width=0.5\linewidth]{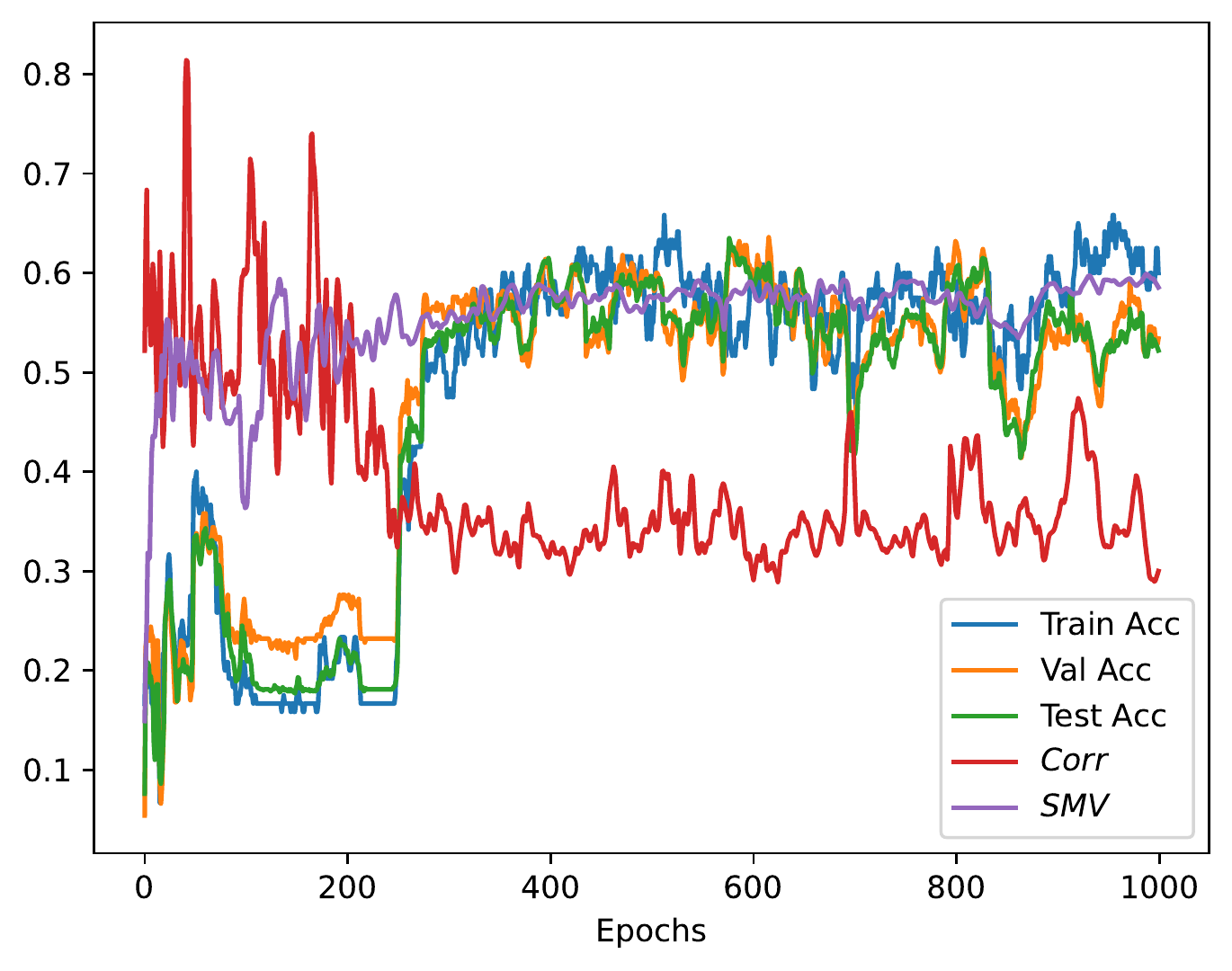}}}%
     \quad
\vskip -1em
\caption{$Corr$, $SMV$ and accuracy over training epochs.}
\vskip -1em
\label{fig:appendix_L15}
\end{figure}

\subsection{Combining with Other Deep GNN Methods}
\label{sec:combining}
In this subsection, we answer the third question and integrate DeCorr into other deep GNN methods. 

\subsubsection{Combining with DGN} 
\label{sec:dgn_ours}
To verify if DeCorr can serve as a complementary technique for methods tackling oversmoothing, we choose the strongest baseline, DGN, to be combined with DeCorr. Specifically, we vary the values of $K$ in $\{2,4,\ldots,20, 25,\ldots,$ $40\}$ and report the test accuracy and $Corr$ values of DGN+DeCorr, DGN and DeCorr on the Cora dataset in Figure~\ref{fig:dgn_ours}. From the figure, we make the following observations:
\begin{compactenum}[(1)]
    \item When combining DGN with DeCorr, we can achieve even better performance than each individual method, which indicates that overcorrelation and oversmoothing are not identical. It provides new insights for developing deeper GNNs as we can combine the strategies tackling overcorrleation with ones solving oversmoothing to enable stronger deeper models.
    \item In Figure~\ref{fig:dgn_ours_corr}, we can observe that DGN is not as effective as DeCorr in reducing feature correlation $Corr$. However, combining DGN with DeCorr can achieve even lower $Corr$ than DeCorr with the larger number of layers. It could be the reason that the training process of DeCorr becomes more stable when combined with DGN, which leads to a better minimization on feature correlation.
\end{compactenum}

\begin{table}[t]
\scriptsize
\caption{Node classification accuracy (\%) for eight datasets.}
\vskip -1em
\begin{tabular}{@{}ccccccccc@{}}
\toprule
 & \textbf{Cora} & \textbf{Cite.} & \textbf{Pubm.} & \textbf{Corn.} & \textbf{Texas} & \textbf{Wisc.} & \textbf{Cham.} & \textbf{Actor}\\ \midrule
 GCN   & 81.5 & 71.1 & 79.0 & 28.2 & 52.7 & 52.2 & 45.9 & 26.9\\
GAT   & 83.1 & 70.8 & 78.5 & 42.9 & 54.3 & 58.4 & 49.4 & 28.5\\
APPNP & 83.3 & 71.8 & 80.1 & 54.3 & 73.5 & 65.4 & 54.3 & 34.5\\ \midrule
GCNII                    & 85.5          & 73.5              & 79.9            & 70.8             & 75.7           & 76.7                & 52.8    & 34.5          \\
GCNII+DeCorr               & \textbf{85.6} & \textbf{73.8}     & \textbf{80.3}   & \textbf{75.4}    & \textbf{76.5}  & \textbf{80.8}       & \textbf{54.1}  & \textbf{35.3}    \\ \midrule
GCNII*                   & \textbf{85.4} & 73.1              & 80.0            & 73.8             & 78.4           & 79.2                & 54.3           & 35.1    \\
GCNII*+DeCorr              & 85.3          & \textbf{73.7}     & \textbf{80.4}   & \textbf{79.2}    & \textbf{80.3}  & \textbf{82.5}       & \textbf{59.0}   & \textbf{35.3}   \\ \bottomrule
\end{tabular}
\vskip -2em
\label{tab:deepmodels}
\end{table}

\subsubsection{Combing with DeepGCN}
We conducted experiments on the layers of DeepGCN~\cite{li2019deepgcns}. As the DeepGCN layer is composed of residual connection (with layer norm) and dilated convolution, we report their performances in the following table separately. Specifically, we report the results of DeepGCN, DeepGCN without dilated convolution (denoted as Res), DeepGCN without residual connection (denoted as DConv) by stacking 15/30 layers. We also report the results of equipping them with DeCorr. The results are summarized in Table~\ref{tab:residual}. From the table, we make three observations: (1) Residual connection can effectively boost the performance and help alleviate overcorrelation issue. This is consistent with our discussion in Section 3.3. (2) Dilated convolution does not help alleviate the issue. It seems that residual connections should be the major reason why DeepGCN can be very deep. (3) Residual connection with DeCorr can achieve even higher performances, e.g., 2.6\%/1.7\% gain for 15/30 layers.

\begin{table}[h]
\small
\caption{DeCorr with DConv/Residual on Cora.}
\label{tab:residual}
\vskip -1em
\begin{tabular}{@{}l|ccc|ccc@{}}
\toprule
               & \multicolumn{3}{c|}{Layer 15} & \multicolumn{3}{c}{Layer 30} \\ \midrule
Methods        & Acc      & Corr    & SMV     & Acc      & Corr    & SMV     \\ \midrule
GCN            & 21.92    & 0.84    & 0.02    & 13.20     & 0.89    & 0.01    \\ 
DConv          & 21.92    & 0.84    & 0.02    & 13.20     & 0.89    & 0.01    \\
Res            & 76.80     & 0.34    & 0.56    & 75.58    & 0.40     & 0.53    \\
DeCorr         & 76.94    & 0.15    & 0.58    & 72.60     & 0.20     & 0.56    \\
DeepGCN        & 76.82    & 0.33    & 0.58    & 75.46    & 0.33    & 0.58    \\
DConv+DeCorr   & 74.62    & 0.27    & 0.56    & 68.06    & 0.36    & 0.50     \\
Res+DeCorr     & 79.40     & 0.25    & 0.59    & 77.30     & 0.22    & 0.60     \\
DeepGCN+DeCorr & 76.48    & 0.25    & 0.62    & 76.56    & 0.23    & 0.63    \\ \bottomrule
\end{tabular}
\vskip -1em
\end{table}

\subsubsection{Combining with other deep models.} In addition to those general frameworks which alleviate the oversmoothing issue, we further combine deep models GCNII and GCNII*~\cite{chen2020simple} (the depth is in $\{16, 32, 64\}$) with DeCorr.  We perform experiments on eight benchmark datasets and report the average accuracy for 10 random seeds in Table~\ref{tab:deepmodels}. As shown in the table, we find that DeCorr can further improve both GCNII and GCNII* in most datasets. For example, Decorr improves GCNII* by 0.6\%, 0.4\%, 5.4\%, 1.9\%, 3.3\% and 4.7\% on Citeseer, Pubmed, Cornell, Texas, Wisconsin and Chameleon, respectively. Such observation further supports that decorrelating features can help boost the model performance.

\begin{figure}[t]%
     \centering
     \subfloat[Test Accuracy]{\label{fig:dgn_ours_acc}{\includegraphics[width=0.5\linewidth]{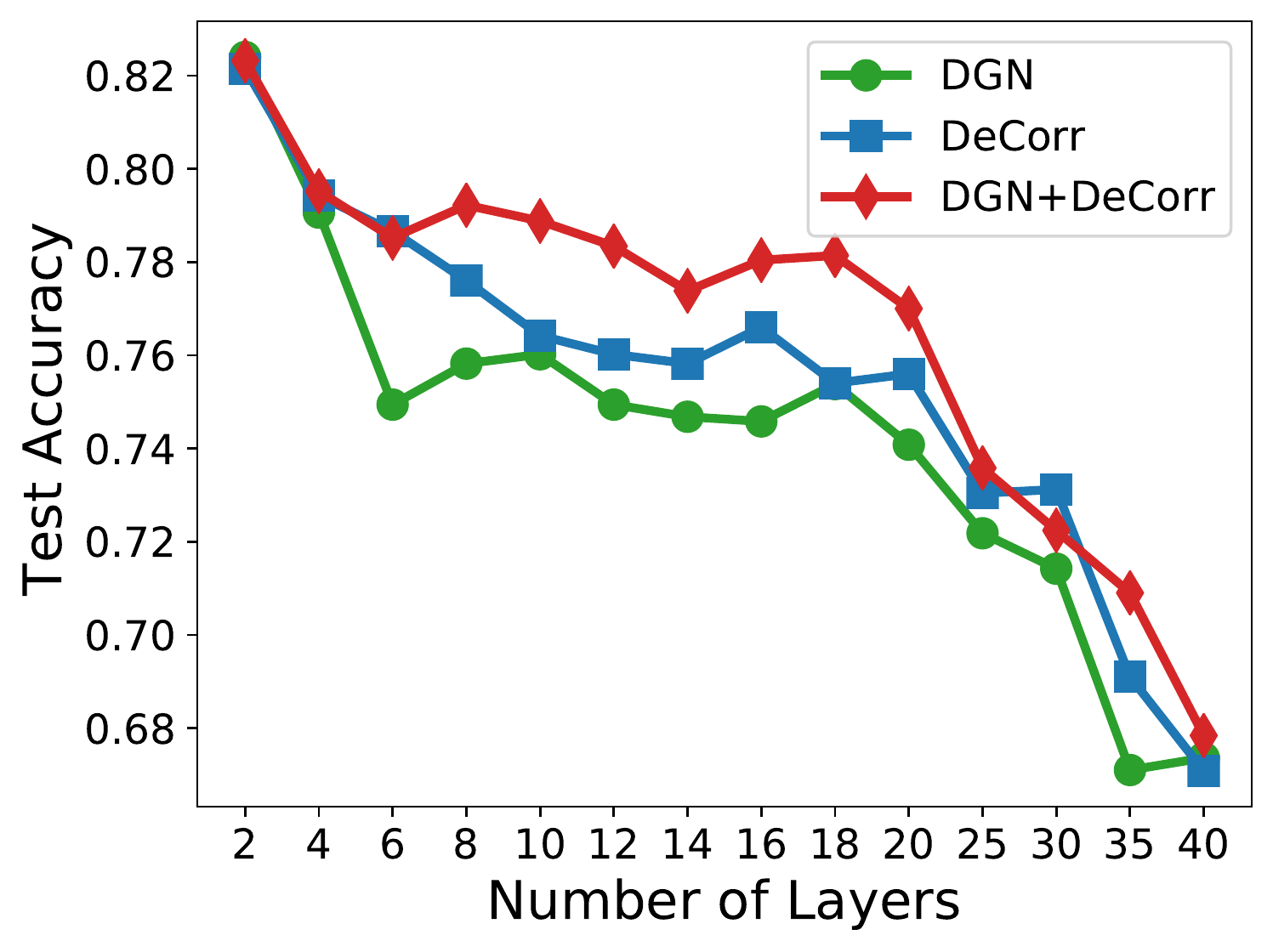} }}%
      \subfloat[Correlation]{\label{fig:dgn_ours_corr}{\includegraphics[width=0.5\linewidth]{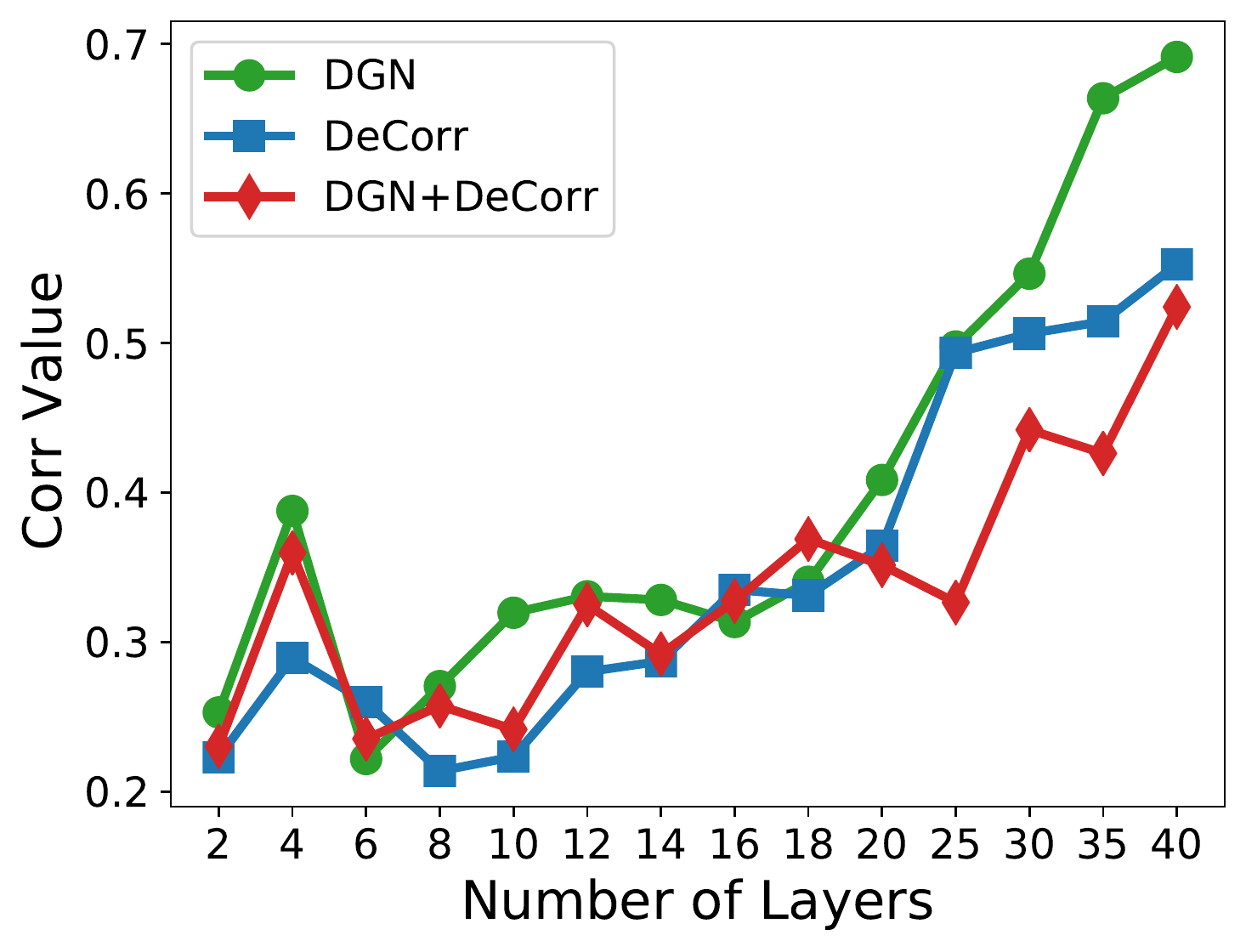} }}%
    \qquad
    \vskip -1em
\caption{Test accuracy and $Corr$ values on the Cora dataset.}
\vskip -1em
\label{fig:dgn_ours}%
\end{figure}

\begin{figure}[t]%
     \centering
     \subfloat[Cora]{{\includegraphics[width=0.5\linewidth]{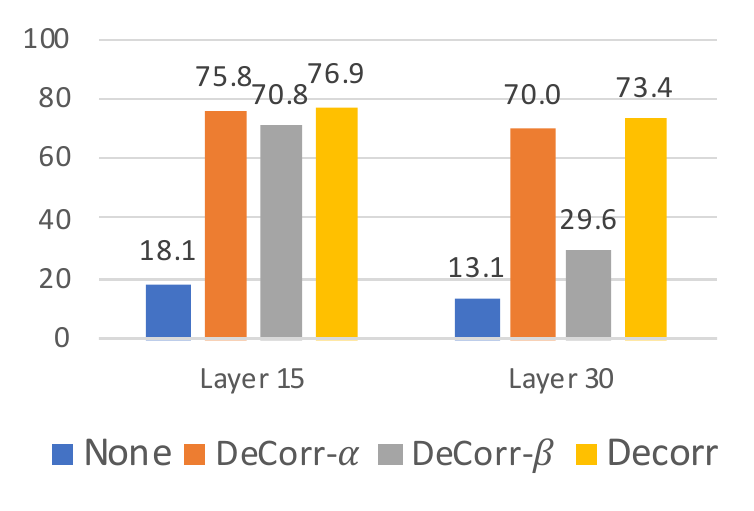} }}%
      \subfloat[Citeseer]{{\includegraphics[width=0.5\linewidth]{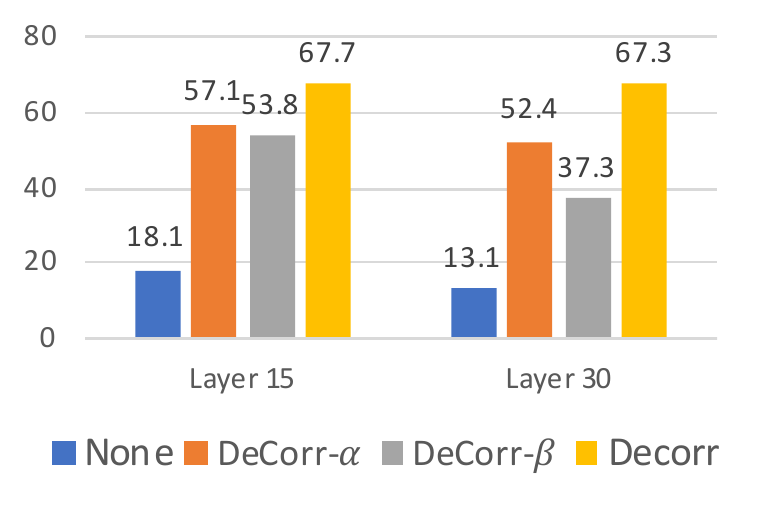} }}%
      \qquad
          \vskip -0.5em
\vskip -1em
\caption{Ablation study for Cora and Citeseer.}
\vskip -1em
\label{fig:ablation}%
\end{figure}


\subsection{Ablation Study}
In this subsection, we take a deeper look at the proposed framework to understand how each component affects its performance and answer the fourth question. Due to space limit, we focus on studying GCN while similar observations can be made for GAT and ChebyNet. We choose 15/30-layer GCN to perform the ablation study on Cora and Citeseer (check Appendix A.2 for Pubmed and CoauthorCS). Specifically, we create the following ablations:
\begin{itemize}
\item \textbf{None:} vanilla GCN without any other components.
\item \textbf{DeCorr-$\alpha$:} we remove the $\mathcal{L}_M$ loss but keep $\mathcal{L}_D$.
\item \textbf{DeCorr-$\beta$:} we remove the $\mathcal{L}_D$ loss but keep $\mathcal{L}_M$.
\end{itemize}
The results are shown in Figure~\ref{fig:ablation}. As shown in the figure, both DeCorr-$\alpha$ and DeCorr-$\beta$ can greatly improve the performance of 15-layer GCN, indicating that optimizing $\mathcal{L}_D$ and $\mathcal{L}_M$ can both help reduce feature correlation and boost the performance. Note that DeCorr-$\alpha $ and DeCorr-$\beta$ can achieve comparable performance on 15-layer GCN.  However, on 30-layer GCN,  DeCorr-$\beta$ does not bring as much improvement as DeCorr-$\alpha$ does on the four datasets. This observation suggests that when GNNs are very deep, explicit feature decorrelation ($\mathcal{L}_D$) is of more significance than the implicit method ($\mathcal{L}_M$). 

\section{Conclusion}
Graph neural networks suffer severe performance deterioration when deeply stacking layers. Recent studies have shown that the oversmoothing issue is the major cause of this phenomenon. In this paper, we introduce a new perspective in deeper GNNs, i.e., feature overcorrelation, and perform both theoretical and empirical studies to deepen our understanding on this issue. We find that overcorrelation and oversmothing present different patterns while they are also related to each other. To address the overcorrelation issue, we propose a general framework, DeCorr, which aims to directly reduce the correlation among feature dimensions while maximizing the mutual in-formation between input and the representations. Extensive experiments have demonstrated  that the proposed DeCorr can help deeper GNNs encode more useful information and achieve better performance under the settings of normal graphs and graphs with missing features. As one future work, we plan to explore the potential of applying DeCorr on various real-world applications such as recommender systems.

\section*{ACKNOWLEDGEMENT}
This research is supported by  the National Science Foundation (NSF) under grant numbers IIS1714741, CNS1815636, IIS1845081, IIS1907704, IIS1928278, IIS1955285, IOS2107215, and IOS2035472, and the Army Research Office (ARO) under grant number W911NF-21-1-0198.

\bibliographystyle{ACM-Reference-Format}
\bibliography{sample}

\newpage
\appendix

\section{Additional Experimental Results}
\label{appendix:exp_res}

\subsection{Preliminary Study}
\label{appendix:pre}
\textbf{Correlation and Smoothness w.r.t. GCN Layers.} We further plot the changes of $Corr$ and $SMV$ w.r.t. number of GCN layers on Pubmed and CoauthorCS in Figure~\ref{fig:appendix_pubmed_cs}. From the figure, we make two observations. (1) With the increase of number of layers, the $Corr$ value tends to increase while the $SMV$ does not change much on these two datasets. (2) The test accuracy drops with the increase of number of layers. Based on the two observations, we can conjecture that it is not oversmoothing but overcorrelation that causes the performance degradation in Pubmed and CoauthorCS.

\begin{figure}[h]%
     \centering
     \subfloat[Pubmed]{{\includegraphics[width=0.5\linewidth]{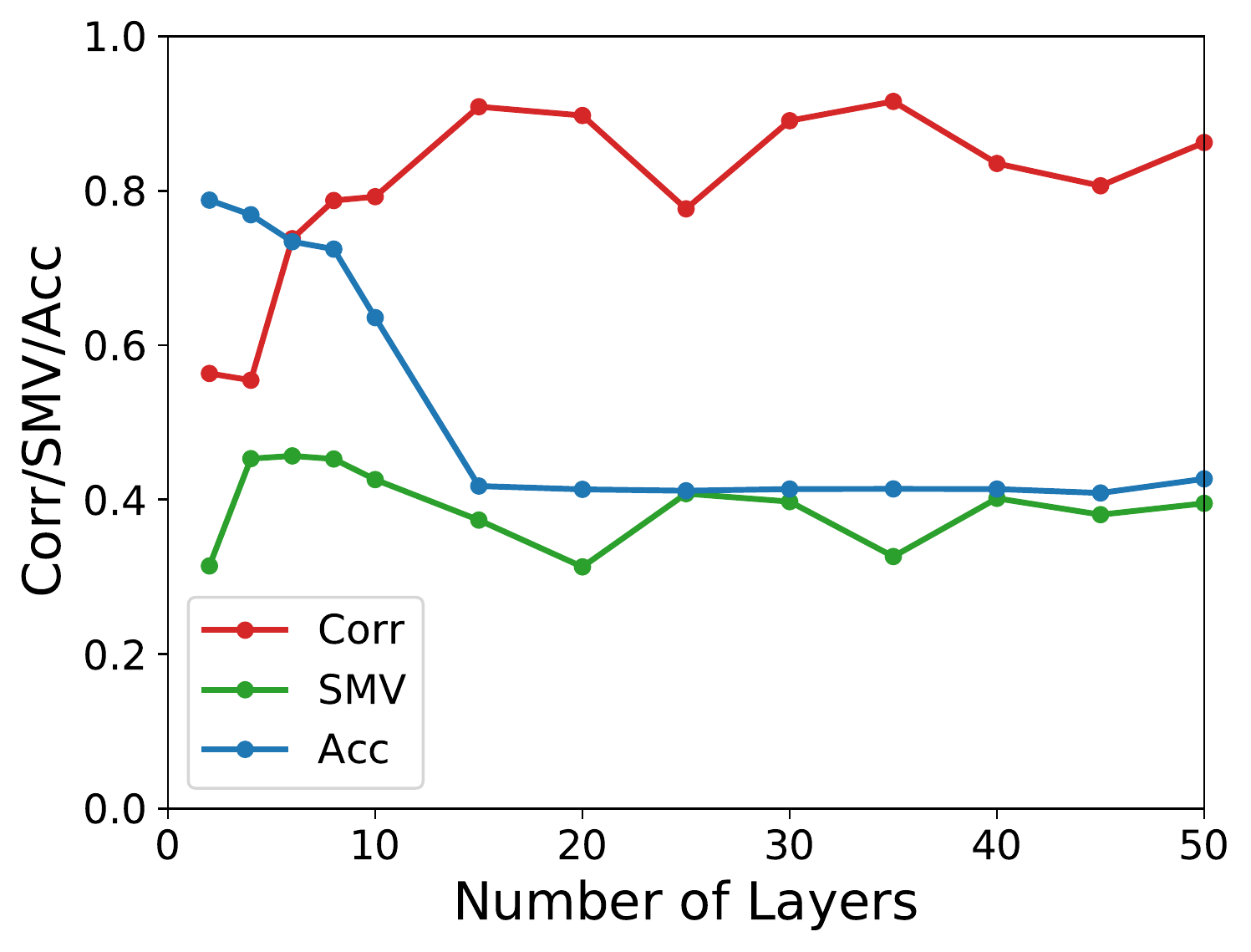}}}%
     \subfloat[CoauthorCS]{{\includegraphics[width=0.5\linewidth]{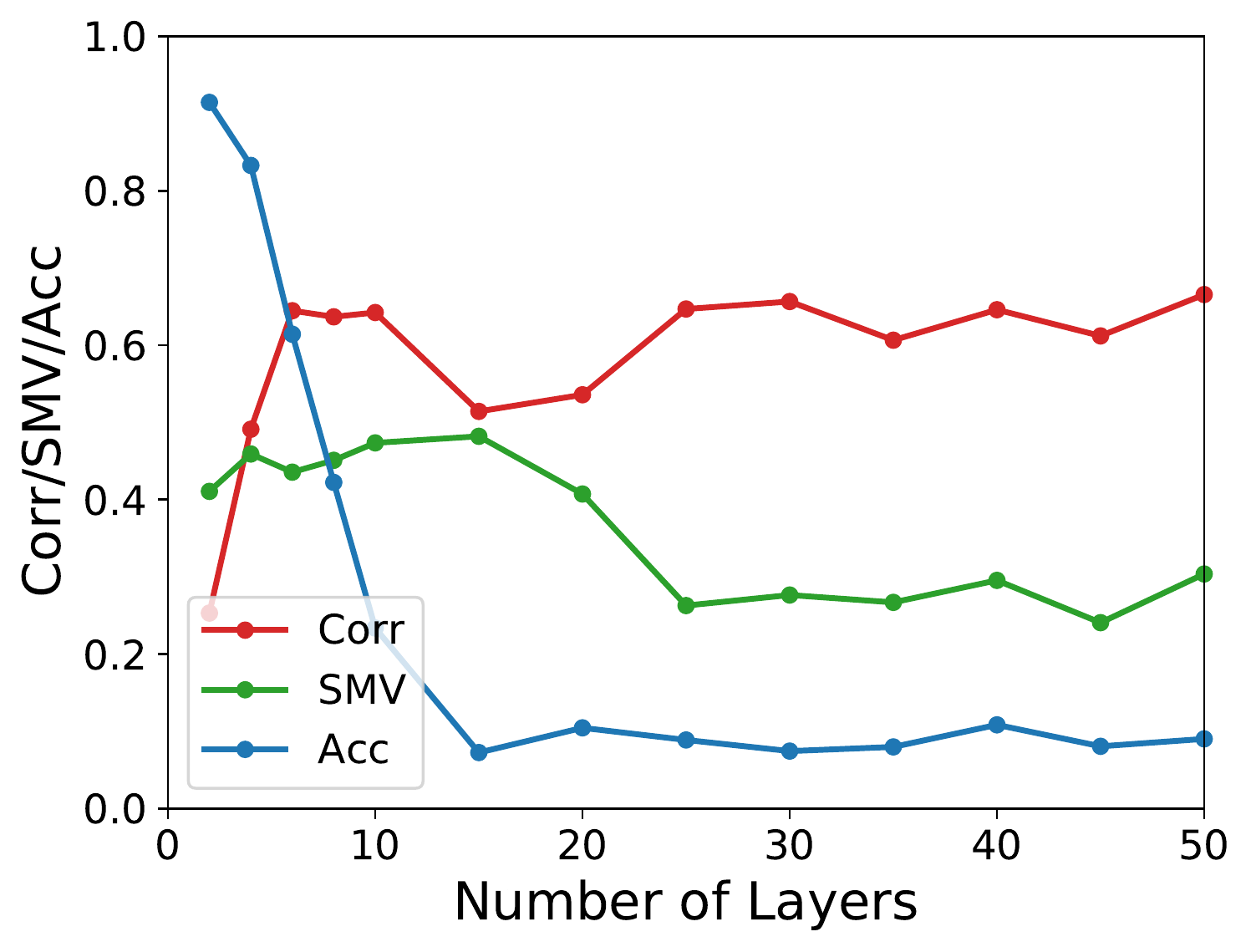}}}%
     \quad
\caption{$Corr$, $SMV$ and test accuracy w.r.t number of GCN layers.}
\label{fig:appendix_pubmed_cs}
\end{figure}

\subsection{Ablation Study}
 As shown in Figure~\ref{fig:ablation_pubmed}, we have similar observations as we made in Cora and Citeseer. Both DeCorr-$\alpha$ and DeCorr-$\beta$ can greatly improve the performance of 15-layer GCN, indicating that optimizing $\mathcal{L}_D$ and $\mathcal{L}_M$ can both help reduce feature correlation and boost the performance. Note that DeCorr-$\alpha $ and DeCorr-$\beta$ can achieve comparable performance on 15-layer GCN.  However, on 30-layer GCN,  DeCorr-$\beta$ does not bring as much improvement as DeCorr-$\alpha$ does on the four datasets. This observation suggests that when GNNs are very deep, explicit feature decorrelation ($\mathcal{L}_D$) is of more significance than the implicit method ($\mathcal{L}_M$). 
 
\begin{figure}[t]%
     \centering
     \subfloat[Pubmed]{{\includegraphics[width=0.5\linewidth]{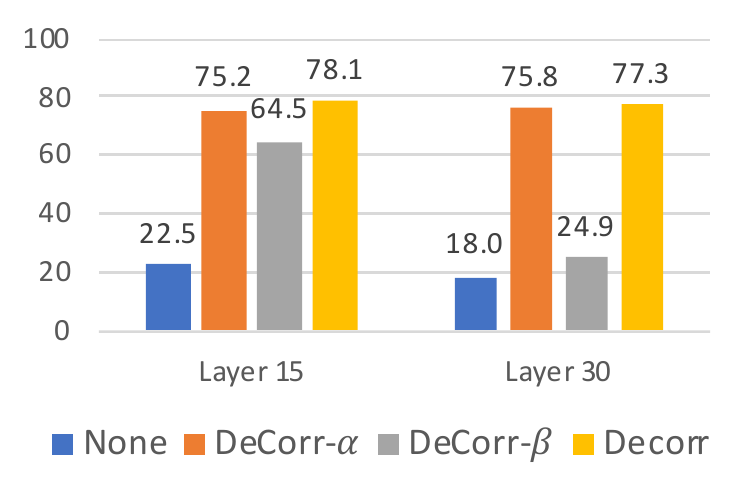}}}%
      \subfloat[CoauthorCS]{{\includegraphics[width=0.5\linewidth]{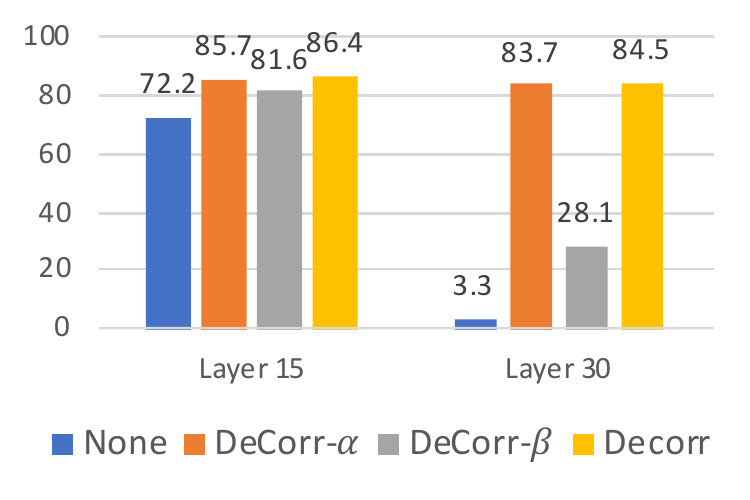} }}%
    \qquad
    \vskip -0.5em
\caption{Ablation study for Pubmed and Coauthors.}
\label{fig:ablation_pubmed}%
\end{figure}

\section{Experiment Setup}
\subsection{Dataset Statistics}
\label{sec:dataset}
The dataset statics is shown in Table~\ref{tab:dataset}. For the experiments on Cora/Citeer/Pubmed, we follow the widely used semi-supervised setting in~\cite{kipf2016semi,zhou2020towards} with 20 nodes per class for training, 500 nodes for validation and 1000 nodes for test. For CoauthorCS, we follow~\cite{zhou2020towards} and use 40 nodes per class for training, 150 nodes per class for validation and the rest for test. For other datasets, we follow~\cite{pei2020geom,chen2020simple} 
to randomly split nodes of each class into 60\%, 20\%, and 20\% for training, validation and test. 

\begin{table}[h]
\vskip -0.5em
\caption{Dataset Statistics.}
\vskip -1em
\begin{tabular}{@{}lccccc@{}}
\toprule
Datasets    & \#Nodes & \#Edges & \#Features & \#Classes  \\
\midrule 
Cora        & 2,708    & 5,429    & 1,433       & 7         \\
Citeseer    & 3,327    & 4,732    & 3,703       & 6         \\
Pubmed      & 19,717   & 44,338   & 500        & 3        \\ 
CoauthorCS  & 18,333    & 81894   & 6805       & 15 \\
Chameleon  & 2,277 & 36,101 & 2,325 & 5 \\
Actor       & 7,600    & 33,544   & 931        & 5         \\
Cornell     & 183     & 295     & 1,703       & 5         \\
Texas       & 183     & 309     & 1,703       & 5         \\
Wisconsin   & 251     & 499     & 1,703       & 5         \\ 
\bottomrule
\end{tabular}
\vspace{-1em}
\label{tab:dataset}
\end{table}

The datasets are publicly available at:
\begin{compactenum}[(1)]
    \item \textbf{Cora/Citeseer/Pubmed}:\\ https://github.com/tkipf/gcn/tree/master/gcn/data
    \item \textbf{CoauthorCS}: \\
    https://github.com/shchur/gnn-benchmark/tree/master/data/npz
    \item 
    \textbf{Others}:\\ https://github.com/graphdml-uiuc-jlu/geom-gcn/tree/master/\\new$\_$data
\end{compactenum}

\subsection{Parameter Settings}
\label{sec:param_setting}
 
\subsubsection{Experiments in Section~\ref{sec:mainresults}} For BN, PairNorm and DGN, we reuse the performance reported in~\cite{zhou2020towards} for GCN and GAT. For ChebyNet, we use their best configuration to run the experiments. For DropEdge, we tune the sampling percent from $\{0.1, 0.3, 0.5, 0.7\}$, weight decay from $\{0,$ 5$e$-4$\}$, dropout rate from $\{0, 0.6\}$ and fix the learning rate to be 0.005. For the proposed DeCorr, following~\cite{zhou2020towards} we use 5$e$-4 weight decay on Cora, 5$e$-5 weight decay on Citeseer and CoauthorCS, 1$e$-3 weight decay on Pubmed. We further search $\alpha$ from $\{0.1, 1\}$, $\beta$ from $\{1,10\}$, learning rate from $\{0.005,0.01,0.02\}$ and dropout rate from $\{0, 0.6\}$.

\subsubsection{Experiments in Section~5.3.1} For all methods, we tune the learning rate from $\{0.005, 0.01, 0.02\}$, dropout rate from $\{0, 0.6\}$.
For DeCorr, we tune $\alpha$ from $\{0.1, 1\}$, $\beta$ from $\{1,10\}$ while for DGN+ DeCorr we tune $\alpha$ from $\{0.05, 0.1, 0.5\}$ and $\beta$ from $\{0.1, 1\}$.

\subsubsection{Experiments in Section~5.3.2} For GCNII and GCNII*, we use their best configuration as reported in~\cite{chen2020simple}. Based on these configurations, we further equip them with DeCorr and tune $\alpha$ from $\{0.01, 0.05, 0.1\}$ and $\beta$ from $\{0.1,1,10\}$.

\subsubsection{Experiments in Section~5.4} For ablation study, we tune learning rate from $\{0.005, 0.01, 0.02\}$, dropout rate $\{0, 0.6\}$, $\alpha$ from $\{0.1,1\}$ and $\beta$ from $\{0.1,1,10\}$

\subsection{Baselines}
We use the following publicly available implementation of baseline methods and deep models:
\begin{itemize}
    \item \textbf{DGN}: https://github.com/Kaixiong-Zhou/DGN/
    \item \textbf{PairNorm}: https://github.com/Kaixiong-Zhou/DGN/
    \item \textbf{BatchNorm}: https://github.com/Kaixiong-Zhou/DGN/
    \item \textbf{DropEdge}: https://github.com/DropEdge/DropEdge
    \item \textbf{GCNII}: https://github.com/chennnM/GCNII/tree/master/PyG
    \item \textbf{APPNP}: https://github.com/rusty1s/pytorch$\_$geometric
\end{itemize}

\end{document}